\documentclass[11pt]{article}
\usepackage[a4paper,top=3cm,bottom=2cm,left=2cm,right=2cm,marginparwidth=1.75cm]{geometry}
\usepackage{fullpage}
\usepackage{graphics} 
\usepackage{epsfig}
\usepackage{amsmath} 
\usepackage{amssymb} 
\usepackage{fontawesome5}
\usepackage{hyperref}
\usepackage{amsthm}
\usepackage{tabularx,booktabs}
\usepackage{mathrsfs}
\usepackage{selectp}
\usepackage[normalem]{ulem}
\usepackage{authblk}
\usepackage{subcaption}
\usepackage{natbib}
\usepackage{algorithm}
\usepackage{algpseudocode}
\usepackage{enumitem}
\usepackage{wrapfig}
\usepackage{scalerel,stackengine}
\stackMath
\newcommand\reallywidehat[1]{%
\savestack{\tmpbox}{\stretchto{%
  \scaleto{%
    \scalerel*[\widthof{\ensuremath{#1}}]{\kern-.6pt\bigwedge\kern-.6pt}%
    {\rule[-\textheight/2]{1ex}{\textheight}}
  }{\textheight}%
}{0.5ex}}%
\stackon[1pt]{#1}{\tmpbox}%
}

\usepackage{xr}
\usepackage{xcolor}

\newboolean{showcomments}
\setboolean{showcomments}{true}
\ifthenelse{\boolean{showcomments}}
{ \newcommand{\mynote}[3]{
		\fbox{\bfseries\sffamily\scriptsize#1}
		{\small$\blacktriangleright$\textsf{\emph{\color{#3}{#2}}}$\blacktriangleleft$}}
	\newcommand{\zzz}[1]{{\setlength{\fboxsep}{2pt}\fcolorbox{black}{yellow}{\textsf{\emph{#1}}}}\xspace}}
{ \newcommand{\mynote}[3]{}
	\newcommand{\zzz}[1]{}}

\usepackage{float}
\usepackage{booktabs}
\usepackage[table]{xcolor}
\definecolor{oursrow}{RGB}{238,246,252}
\usepackage[toc,page,header]{appendix}
\usepackage{minitoc}

\usepackage{hyperref}
\hypersetup{       
    unicode=false,          
    pdftoolbar=true,       
    pdfmenubar=true,      
    pdffitwindow=false,     
    pdfstartview={FitH},   
    pdfnewwindow=true,      
    colorlinks=true,      
    linkcolor=blue,          
    citecolor=blue,       
    filecolor=magenta,      
    urlcolor=blue,           
    breaklinks=true
}

\newtheorem{theorem}{Theorem}
\newtheorem{assumption}{Assumption}
\newtheorem{lemma}{Lemma}
\newtheorem{sublemma}{Lemma}[lemma]
\newtheorem{definition}{Definition}

\usepackage{threeparttable}
\usepackage{hyperref}
\usepackage{url}
\usepackage{graphicx}
\usepackage{float}
\usepackage{enumitem}

\usepackage[toc,page,header]{appendix}
\usepackage{minitoc}

\title{\bf Preserving Unstable Modes Through Inverse Dynamics in JEPA World Models}
\date{}

\author[1]{Leonardo F. Toso\footnote{Correspondence to: \texttt{leonardo.toso@columbia.edu}.}}
\author[3]{Yann LeCun}
\author[1]{James Anderson}
\author[2]{Oumayma Bounou}
\affil[1]{Columbia University}
\affil[2]{New York University}
\affil[3]{New York University, AMI Labs}

\begin{document}

\doparttoc 
\faketableofcontents

\maketitle

\vspace{-1.5cm}
\begin{center}
    \href{https://jepa-control.github.io/}{%
        \faGlobe\enspace Website%
    }
    \qquad
    \href{https://github.com/LeoToso/JEPA_control}{%
        \faGithub\enspace Code%
    }
\end{center}

\begin{abstract}
Robotic systems often exhibit unstable modes, along which small perturbations and disturbances can cause unbounded growth unless corrected through feedback. Controlling such systems from high-dimensional visual observations requires representations that preserve these modes. Joint-embedding predictive architectures (JEPAs) provide a natural framework for learning such representations and their dynamics from visual data. However, we demonstrate that next step prediction combined with anti-collapse regularization does not guarantee that controllable unstable modes are preserved: the training loss can be minimized while these modes are collapsed, making stabilization from the learned representation impossible. To address this, we augment world-model training with an action reconstruction objective (i.e., an inverse dynamics loss) that encourages control-aware representations, namely, visual representations that preserve crucial features for control. We prove that exact action reconstruction makes the encoder injective on the finite-horizon reachable subspace. Thus, the encoder cannot discard any state direction reachable by an action sequence within $H$ steps. Moreover, we show that, as $H$ grows, the dominant eigenspace of the finite-horizon controllability Gramian converges to the controllable unstable subspace. We establish our theoretical results for linear systems and demonstrate empirically that our findings extend to nonlinear visual control tasks (CartPole, Walker2D, and PointMaze), highlighting the benefits of control-aware representation learning.
\end{abstract}

\allowdisplaybreaks

\section{Introduction}

Quadrotors and legged robots operate around unstable equilibria where small perturbations can grow catastrophically without corrective feedback~\citep{mellinger2011minimum,ames2014rapid}. Controlling such systems requires actions that not only \emph{minimize} a cost (e.g., goal reaching), but also \emph{stabilize} the physical dynamics; otherwise, incorrect feedback can rapidly drive the system from its intended operating regime and compromise safe deployment~\citep{ames2016control,marco2021robot}. 

\begin{figure}[t]
\centering
\includegraphics[width=0.9\textwidth]{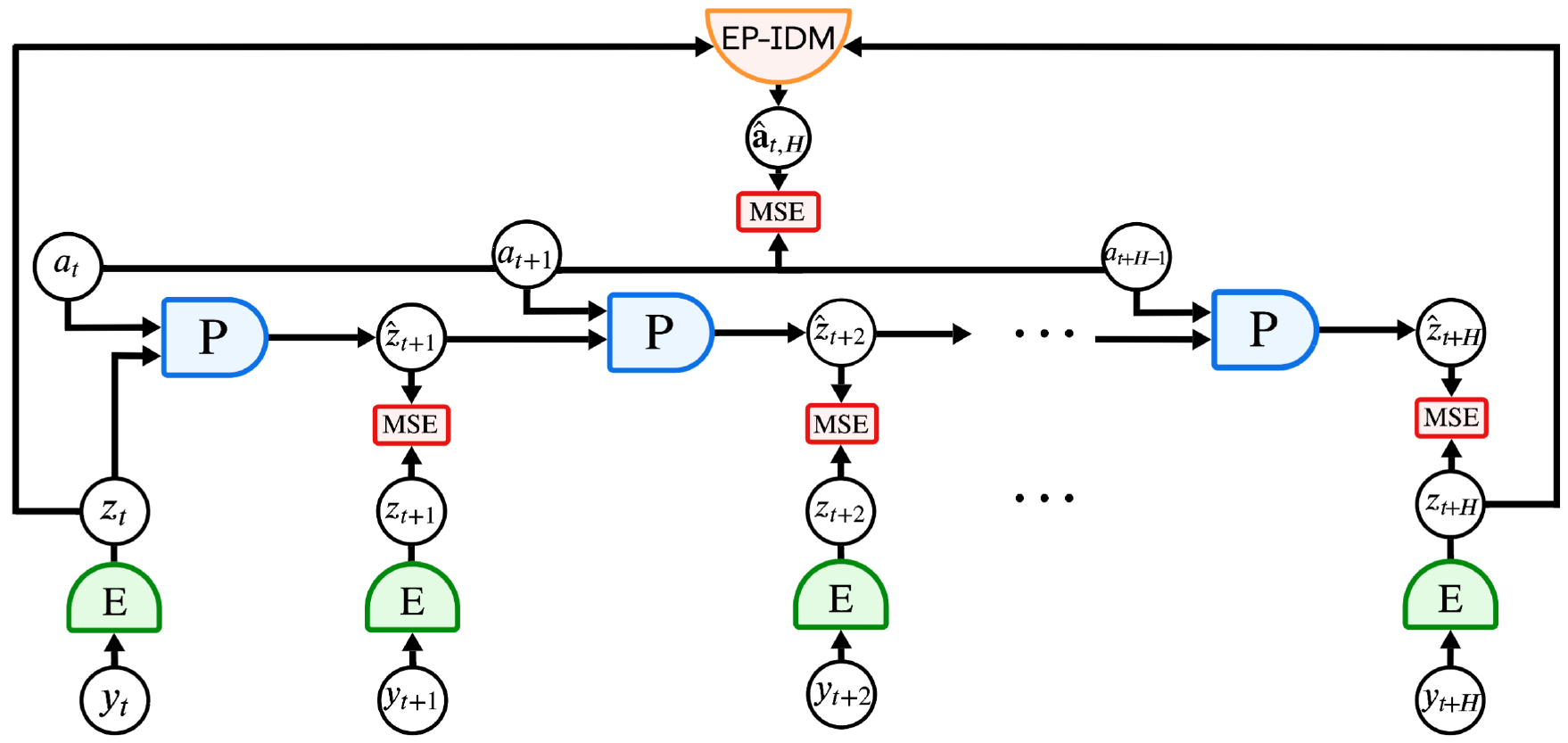}
\caption{\small
Architecture of our proposed world model. At each step, an encoder ($E$)
maps observation $y_t$ to a latent state $z_t$, and a predictor ($P$)
rolls out latent predictions
$\hat{z}_{t+1}, \ldots, \hat{z}_{t+H}$ conditioned on actions. Training losses are applied at each predicted latent: 1SP/MSP
(single- or multi-step prediction) aligns $\hat{z}_{t+k}$ with the
encoded ground-truth $z_{t+k}$. The EP-IDM estimates the intermediate actions
$\hat{\mathbf{a}}_{t,H}
= (\hat a_t, \hat a_{t+1}, \ldots, \hat a_{t+H-1})$ from $z_t$ and $z_{t+H}$.
}
\label{fig:architecture}
\end{figure}

Latent world models have found their way into robotics as a computationally efficient solution for learning dynamics and designing control actions from high-dimensional sensory observations. They compress pixel observations into a lower-dimensional representation, learn action-conditioned latent dynamics, and optimize actions by rolling out these dynamics, with limited interaction with the physical system \citep{ha2018world,hafner2019learning,hafner2019dream,hafnermastering,hafner2023mastering}. Joint-embedding predictive architectures (JEPAs) follow this recipe without pixel-level reconstruction \citep{lecun2022path,sobal2025learning,balestriero2025lejepa}: an encoder extracts a latent representation, a predictor forecasts its evolution, and anti-collapse mechanisms (e.g., SIGReg~\citep{balestriero2025lejepa}, EMA~\citep{assran2023self}, and others~\citep{bardes2021vicreg}) prevent representation collapse. Planning can then be performed efficiently directly in this learned latent space \citep{zhou2024dino,toso2026learning,wang2026temporal}. However, deploying JEPAs on unstable systems remains largely unexplored.

When actions are designed using the learned latent representation and dynamics, the representation must preserve the physical information required for feedback control.
For open-loop unstable systems, this requires preserving the unstable modes that feedback must observe and correct
(Lemma~\ref{lemma:conditions_for_stabilization})
\citep{hu2022sample,werner2024sample, toso2025learning,lutkus2025latent}. That is, if a controller cannot ``see'' an unstable mode, it cannot stabilize it.

However, learning to accurately predict latent dynamics while preventing representation collapse does not guarantee that this control-theoretic information is preserved.
In fact, the objective combining next-step prediction and anti-collapse regularization, the standard recipe for JEPAs, can be minimized while the encoder discards every unstable mode (Lemma~\ref{lemma:collapse_of_unstable_modes}). 
To preserve these modes, we augment the JEPA predictive learning objective with an ``endpoint inverse-dynamics'' (EP-IDM) loss, which reconstructs the underlying action sequence from the initial and final latent states (see Figure \ref{fig:architecture}), so that the encoder cannot discard directions the actions can reach \citep{mhammedi2020learning,ivashkov2026sensorimotor}.
Our contributions are summarized below.

\begin{itemize}[leftmargin=*, itemsep=2pt, topsep=2pt]
\item \textbf{Prediction does not guarantee stabilization.} We show that a controller acting on the latent states can stabilize the system \emph{if and only if} the encoder keeps \emph{every} unstable mode (Lemma~\ref{lemma:conditions_for_stabilization}). We then show, with SIGReg as the anti-collapse regularizer, that the training objective can reach its minimum while the encoder discards \emph{all} unstable modes (Lemma~\ref{lemma:collapse_of_unstable_modes}).

\item \textbf{Inverse dynamics preserves unstable modes.} We introduce EP-IDM, which reconstructs the action sequence from the initial and final latent states. We prove that, with sufficient action excitation, exact reconstruction keeps every direction reachable within $H$ steps, and hence every unstable mode reachable within $H$ steps (Theorem~\ref{theorem:condition_achieved_action_reconstruction}).

\item \textbf{Unstable subspace recovery.} We prove that, as $H$ grows, the dominant eigenspace of the finite-horizon controllability Gramian converges to the unstable subspace (Theorem~\ref{theorem:unstable_subspace_recovery}).

\item \textbf{Numerical validation.} We show on synthetic linear systems and a linearized CartPole that EP-IDM preserves the unstable direction and enables latent stabilization, while SIGReg does not. We further validate the benefits of EP-IDM on nonlinear visual control tasks (CartPole, Walker2D, and PointMaze), where representations and dynamics are learned from images and proprioceptive measurements.
\end{itemize}

\section{Related Work}

Our work lies at the intersection of joint-embedding predictive world models and latent feedback stabilization. 

\vspace{0.2cm}

\noindent \textbf{Joint embedding predictive world models.} JEPAs learn representations by predicting future latent states rather than reconstructing pixels \citep{lecun2022path,assran2023self}. Recent work applies this recipe to action-conditioned prediction and planning using pre-trained visual encoders \citep{zhou2024dino,toso2026learning}, end-to-end anti-collapse mechanisms \citep{balestriero2025lejepa,maes2026leworldmodel,kuang2026lpwm}, or objectives that shape latent dynamics and geometry \citep{sobal2025learning,parthasarathy2025closing,wang2026temporal,zhang2026hierarchical,nath2026pixels}. While this body of work primarily addresses total representation collapse and planning, we demonstrate that standard JEPA objectives can selectively collapse directions required for feedback stabilization.

Most closely related,~\citet{ivashkov2026sensorimotor} use one-step inverse dynamics to prevent collapse. We share the principle that actions provide a control-relevant learning signal but we study feedback stabilization of open-loop unstable systems. Our EP-IDM loss reconstructs the entire action sequence from the endpoint latent states. We prove that exact reconstruction preserves the finite-horizon reachable subspace and characterize its convergence to the controllable unstable subspace, connecting action reconstruction to detectability and latent stabilizability.
\vspace{0.2cm}

\noindent \textbf{Latent feedback stabilization.} Control from visual observations requires retaining the state information needed for stabilization. Prior work studies stabilization of linear systems under unknown nonlinear observation maps \citep{mhammedi2020learning}, characterizes the sample complexity of learning to stabilize \citep{hu2022sample,werner2024sample}, and learns low-dimensional unstable subspace representations \citep{toso2025learning,lutkus2025latent}. In particular, preserving only the unstable subspace can suffice, since stable modes already decay without correction \citep{toso2025learning}. More than just constructing those sufficient representations, we also ask whether standard JEPA training can learn them in the first place. In particular, we connect latent world model training objectives to the control-theoretic information that their representations must preserve.

\section{Setup}

We introduce the dynamical system, latent world model, and training objectives considered throughout the paper.
\vspace{0.2cm}

\noindent \textbf{Learning an encoder and a predictor. } We consider the discrete-time dynamical system
\begin{equation}
\label{eq:nonlin_system}
x_{t+1} = f(x_t,a_t), \qquad
y_t = g(x_t) \in \mathcal Y \subset \mathbb R^p,
\end{equation}
for all $t=0, 1,2, \dots $, where $x_t\in\mathcal X\subset\mathbb R^n$ is the state and $a_t\in\mathcal A\subset\mathbb R^m$ is the applied action, at time $t$. The transition map $f:\mathcal X\times\mathcal A\to\mathcal X$ describes the system dynamics, while the observation map $g:\mathcal X\to\mathcal Y$ maps the state to the observation $y_t$. This observation may contain images, proprioceptive measurements (i.e., part or all of the state $x_t$), both, or other modalities \citep{wang2026music,huang2026llm}. Examples of such a mapping are provided in Appendix~\ref{appendix:example_details}. We call Eq.~(\ref{eq:nonlin_system}) the physical system.

Our goal is to learn an observation encoder $E_\theta:\mathcal Y\to\mathcal Z$, an action encoder $G_\theta:\mathcal A\to\mathcal U$, and a latent dynamics model (predictor) $P_\theta:\mathcal Z\times\mathcal U\to\mathcal Z$ such that
\begin{equation}
z_{t+1}= P_\theta\bigl(z_t,G_\theta(a_t)\bigr) \text{ with latent state } z_t=E_\theta(y_t),
\end{equation}
where $\mathcal Z\subset\mathbb R^d$ is the latent space, with $d \ll p$. For notation simplicity, we denote by $\theta\in\mathbb R^{d_\theta}$ the trainable parameters of the observation encoder, action encoder, and predictor.
\vspace{0.2cm}

\noindent \textbf{Standard training objective.} Let $\mathcal D$ denote the distribution of finite-horizon trajectory
windows collected from the system in Eq.~\eqref{eq:nonlin_system}. Each
sampled window $\tau\sim\mathcal D$ starts at a trajectory-dependent time
$t_\tau\geq0$ and is given by
\begin{align}
\tau := \left(y_{t_\tau-1:t_\tau+H}, s_{t_\tau:t_\tau+H}, a_{t_\tau:t_\tau+H-1}\right),
\end{align}
where $y_t$, $s_t$, and $a_t$ denote the image observation, proprioceptive measurement, and action at time $t$, respectively. Hence, each window has $H+2$ images, $H+1$ proprioceptive
measurements, and $H$ actions. The preceding image $y_{t_\tau-1}$ is
used only to construct the first encoder input, resulting in $H+1$
latent states and $H$ prediction transitions.  In our experiments, the encoder uses two consecutive images, their difference to capture motion,
and the current proprioceptive measurement:
\begin{align}
z_{t_\tau+k} = E_\theta \left( y_{t_\tau+k-1}, y_{t_\tau+k}, y_{t_\tau+k}-y_{t_\tau+k-1}, s_{t_\tau+k}
\right), \qquad k=0,\ldots,H.
\end{align}
We train the model using the next-step prediction objective
\begin{equation}
\mathcal L_{\mathrm{pred}}(\theta)
=\mathbb E_{\tau\sim\mathcal D}\left[
\frac{1}{H}\sum_{k=0}^{H-1}
\left\|P_\theta\bigl(z_{t_\tau+k},G_\theta(a_{t_\tau+k})\bigr)
-z_{t_\tau+k+1}\right\|_2^2
\right].
\end{equation}
We combine the prediction loss with SIGReg~\citep{balestriero2025lejepa}
to prevent total representation collapse:
\begin{equation}
\label{optimization:ms_SIGreg}
\min_\theta \; \mathcal L_{\mathrm{pred}}(\theta)
+ \lambda_{\mathrm{reg}}\mathcal L_{\mathrm{reg}}(\theta),
\end{equation}
where $\lambda_{\mathrm{reg}}>0$ controls the strength of the regularizer.
\vspace{0.2cm}

\noindent \textbf{Total representation collapse.}  The representation collapse in JEPAs occurs when the encoder $E_\theta$ maps almost all observations to the same latent representation, i.e.,
$E_\theta(y)=c$ almost surely for some constant $c$, so that the latent
distribution has zero covariance. In particular, SIGReg \citep{balestriero2025lejepa} (also defined in Appendix \ref{appendix:SIGreg_loss} for completeness) avoids such
a degenerate solution by encouraging the aggregate latent distribution to
remain dispersed and approximately isotropic. However, this condition
does not require the encoder to preserve any particular physical direction: the representation can keep sufficient variability while selectively discarding the information (i.e., unstable modes) required for feedback
stabilization. We formalize this collapse of the unstable modes in Section~\ref{sec:stabilization_conditions}.

\vspace{0.2cm}

\noindent \textbf{Planning in the latent space. }\footnote{We emphasize that this goal-reaching planning cost is deliberately simple and may be insufficient for more complex control tasks, which can require task-specific stage costs, learned latent reward models, or costs defined through a state decoder (as done for the Walker2D task in Section \ref{sec:numerics}).} Given an initial observation $y_0$ and a goal observation $y_{\mathrm g}$, we encode $z_0=E_\theta(y_0)$ and $z_{\mathrm g}=E_\theta(y_{\mathrm g})$ and optimize the action sequence through the learned latent dynamics. For a planning horizon $T$, we solve
\begin{equation}
\begin{aligned}
\min_{\boldsymbol a}\quad &\|z_T-z_{\mathrm g}\|_2^2 \qquad \text{subject to}\quad
&z_{t+1}=P_\theta\bigl(z_t,G_\theta(a_t)\bigr),
\qquad t=0,\ldots,T-1,
\end{aligned}
\end{equation}
where $\boldsymbol a=(a_0,\ldots,a_{T-1})$ is the designed action sequence. In practice, we use a receding-horizon implementation. More details in Appendix \ref{appendix:experimental_details}. 
\vspace{0.2cm}

\noindent \textbf{Open-loop stability.}
Let $x_{\mathrm e}$ be an equilibrium under zero input, i.e., $x_{\mathrm e}=f(x_{\mathrm e},0)$. The system is \emph{locally open-loop stable} around $x_{\mathrm e}$ if trajectories initialized sufficiently close to $x_{\mathrm e}$ remain close to it when $a_t=0$ for all $t$. It is \emph{locally asymptotically stable} if these trajectories additionally converge to $x_{\mathrm e}$ as $t\rightarrow \infty$, and it is \emph{open-loop unstable} if the trajectories diverge from $x_{\mathrm e}$. 

For a linear system $x_{t+1}=Ax_t$, 
open-loop asymptotic stability is equivalent to $A$ being Schur stable, i.e., $\rho(A)<1$. Eigenvalues satisfying $|\lambda|>1$ correspond to unstable modes, while those with $|\lambda|=1$ are marginally stable. We emphasize that Schur stability prevents $\|x_t\|\rightarrow \infty$ as $t\rightarrow \infty$. Likewise, stabilization allows us to design the feedback that sends $\|x_t\|$ to zero.
\vspace{0.2cm}

\noindent \textbf{Stabilization and unstable modes.} For an open-loop unstable system, i.e., $x_t=Ax_t$ with $\rho(A)>1$, the goal is to design a feedback controller that selects actions from the available observations and makes the desired equilibrium asymptotically stable. Under  a static linear feedback controller $a_t = Kx_t$, this amounts to finding $K$ such that the closed-loop matrix $A + BK$ is Schur stable.
When the controller acts only on the learned latent state, this is possible only if the encoder preserves every unstable or marginally stable direction that must be corrected through feedback \citep{toso2025learning}. In this work, we demonstrate that this requirement is not implied by an accurate latent next-step predictor alone. 

Next, we formalize necessary and sufficient conditions for latent stabilization and demonstrate that the objective above in Eq. (\ref{optimization:ms_SIGreg}) does not necessarily satisfy them.

\section{Prediction Does Not Guarantee Stabilization}
\label{sec:stabilization_conditions}

An encoder $E_\theta$ can discard unstable modes while still learning
predictable, non-constant representations. In particular, there may exist an initial condition and a controller that stabilizes the learned latent dynamics, so that $\|z_t\|_2\to0$, while the corresponding physical trajectory satisfies $\|x_t\|_2\to\infty$ because an unstable mode lies in the kernel of the encoder. In the linear setting, we characterize the conditions that rule out this failure and demonstrate that next step prediction combined with SIGReg does not necessarily satisfy them. We illustrate these findings at the end of this section with two examples (Fig.~\ref{fig:Examples_main}) and further validate them in
Section~\ref{sec:numerics}, with additional details also provided in
Appendix~\ref{appendix:example_details}.
\vspace{0.2cm}

\noindent \textbf{Linear setting.} We consider a linear dynamical system (LDS) obtained by linearizing Eq.(\ref{eq:nonlin_system}) around an equilibrium point $x^\star$, and we assume $x^\star = 0$ without loss of generality: 
\begin{align}\label{system_LDS}
x_{t+1} &= A x_t + B a_t, \qquad
y_t = C x_t,
\end{align}
for all $t \in \mathbb Z_+$, where $x_t \in \mathbb R^n$, $a_t \in \mathbb R^m$, $y_t \in \mathbb R^p$, and $A$, $B$,
and $C$ are the transition,
control, and observation matrices, respectively. We consider an open-loop unstable system (i.e., $\rho(A) > 1$).

\begin{definition}[Stabilizability, observability, and detectability]
The pair $(A,B)$ is \emph{stabilizable} if there exists a gain $K$ such that $A+BK$ is Schur stable under $a_t=Kx_t$. Given $q_t=Mx_t$, the pair $(A,M)$ is \emph{observable} if $x_0$ can be recovered from finitely many observations, and \emph{detectable} if every unobservable mode is Schur stable.  By the Popov-Belevitch-Hautus (PBH) test~\citep{hautus_controllability}, detectability is also equivalent to
\begin{align}
\label{eq:pbh_detectability}
\operatorname{rank}
\begin{bmatrix}
\lambda I-A\\
M
\end{bmatrix}
=n
\quad
\text{for every }
\lambda\in\operatorname{spec}(A)
\text{ such that }|\lambda|\geq1,
\end{align}
where $\operatorname{spec}(A)$ denotes the set of eigenvalues of $A$.
\end{definition}

\begin{assumption}
\label{assumption:stabilizable_observable}
For the system (\ref{system_LDS}), the pair $(A,B)$ is stabilizable and  $(A,C)$ is observable.
\end{assumption}

We consider an affine observation encoder and a linear predictor:
\begin{align}
E_\theta(y) = W y + b = z \in \mathcal Z, \qquad
P_\theta(z, a) = A_z z + B_z a \in \mathcal Z. 
\end{align}
Here $d<n$. The vector $\theta=(W,b,A_z,B_z)$ collects all trainable parameters. We write the linear predictor directly in terms of the actions, as a linear action encoder can be absorbed into $B_z$ without loss of expressivity.

\subsection{Necessary and Sufficient Conditions for Stabilization of LDS}

We first characterize the conditions under which a controller acting only on the learned latent state can stabilize the physical system. We defer the proof of Lemma \ref{lemma:conditions_for_stabilization} to Appendix
\ref{appendix:proof_conditions_for_stabilization}. 

\begin{lemma}
\label{lemma:conditions_for_stabilization}
Let $F=WC$ denote the state-to-latent map and consider the centered latent state $\tilde z_t=z_t-b=Fx_t$. Suppose that Assumption~\ref{assumption:stabilizable_observable} holds. Then there exists a causal dynamic controller that uses only the centered latent states $\{\tilde z_\tau\}_{\tau=0}^{t}$ and asymptotically stabilizes Eq.~(\ref{system_LDS}) if and only if $(A,F)$ is detectable. Equivalently, every unstable or marginally stable mode of $A$ is observable through $F$, or, by the PBH test, there does not exist a nonzero vector $v\in\mathbb C^n$ such that
\begin{align}
Av=\lambda v,\qquad Fv=0,\qquad |\lambda|\geq1.
\end{align}
\end{lemma}

Lemma~\ref{lemma:conditions_for_stabilization} demonstrates that unstable and marginally stable modes  in~\eqref{system_LDS} need to remain observable through the encoder $E_\theta$, i.e., $E_\theta$ needs to preserve them. In the next subsection, we demonstrate that minimizing next step prediction loss together with SIGReg does not guarantee this property.

\subsection{The Collapse of Unstable Modes}
Let $\Phi_{\geq1}$ denote the unstable subspace of $A$
associated with eigenvalues satisfying $|\lambda|\geq1$. We next demonstrate that standard predictive training objectives (e.g., Eq.~(\ref{optimization:ms_SIGreg})) can discard the entire subspace $\Phi_{\geq1}$ while still reaching their minimum.

\begin{lemma}
\label{lemma:collapse_of_unstable_modes}
There exist open-loop unstable linear systems satisfying Assumption~\ref{assumption:stabilizable_observable} and data distributions for which the prediction loss $\mathcal L_{\mathrm{pred}}$ and SIGReg objective are simultaneously minimized by an encoder satisfying $\Phi_{\geq1} \subset \ker(F)$. 
\end{lemma}\vspace{-0.2cm}
\begin{proof}
We begin our proof by setting the encoder bias to zero and consider coordinates in which
\begin{align}
A=\begin{bmatrix}A_u&\Delta\\0&A_s\end{bmatrix}, \; B=\begin{bmatrix}B_u\\B_s\end{bmatrix}, \text{ and }
x_t=\begin{bmatrix}x_{u,t}\\x_{s,t}\end{bmatrix}.
\end{align}  
Here, we note that $A_u$ is open-loop unstable and $A_s$ is stable. We choose $F=[0\ F_s]$, where $F_s$ is invertible on the stable coordinates. Then, we have 
\begin{align}
z_{t+1} = F_s A_s F_s^{-1} z_t + F_s B_s u_t.
\end{align}
Therefore, choosing $A_z = F_s A_s F_s ^{-1}$ and $B_z = F_s B_s$ implies zero one-step error, while every unstable direction lies in $\ker(F)$.

Note that SIGReg~\citep{balestriero2025lejepa} does not rule out this degenerate solution. Let $x_s \sim \mathcal{N}(0,\Sigma_s)$, with $\Sigma_s \succ 0$, and take
$F_s = Q\Sigma_s^{-1/2}$ for any orthogonal matrix $Q$. Then $z_t$ is distributed according to a standard Gaussian. We then note that when we use population SIGReg objective in Eq.~(\ref{optimization:ms_SIGreg}) the regularization term $\mathcal L_{\mathrm{reg}}$ can be minimized despite the unstable directions being absent in the latent space. 
\end{proof}

The constructed encoder $E_\theta$ minimizes both terms of Eq.~\eqref{optimization:ms_SIGreg} and therefore minimizes Eq.~\eqref{optimization:ms_SIGreg} for every $\lambda_{\mathrm{reg}}\geq0$, while discarding the entire unstable subspace. By Lemma~\ref{lemma:conditions_for_stabilization}, the physical system cannot be stabilized from the resulting latent state. Our linear analysis removes representation capacity as a confounding factor: even with a linear observation-to-latent map and little compression, standard predictive objectives can discard precisely the state directions required for stabilization. 

\begin{figure}[t]
\centering
\includegraphics[width=1\textwidth]{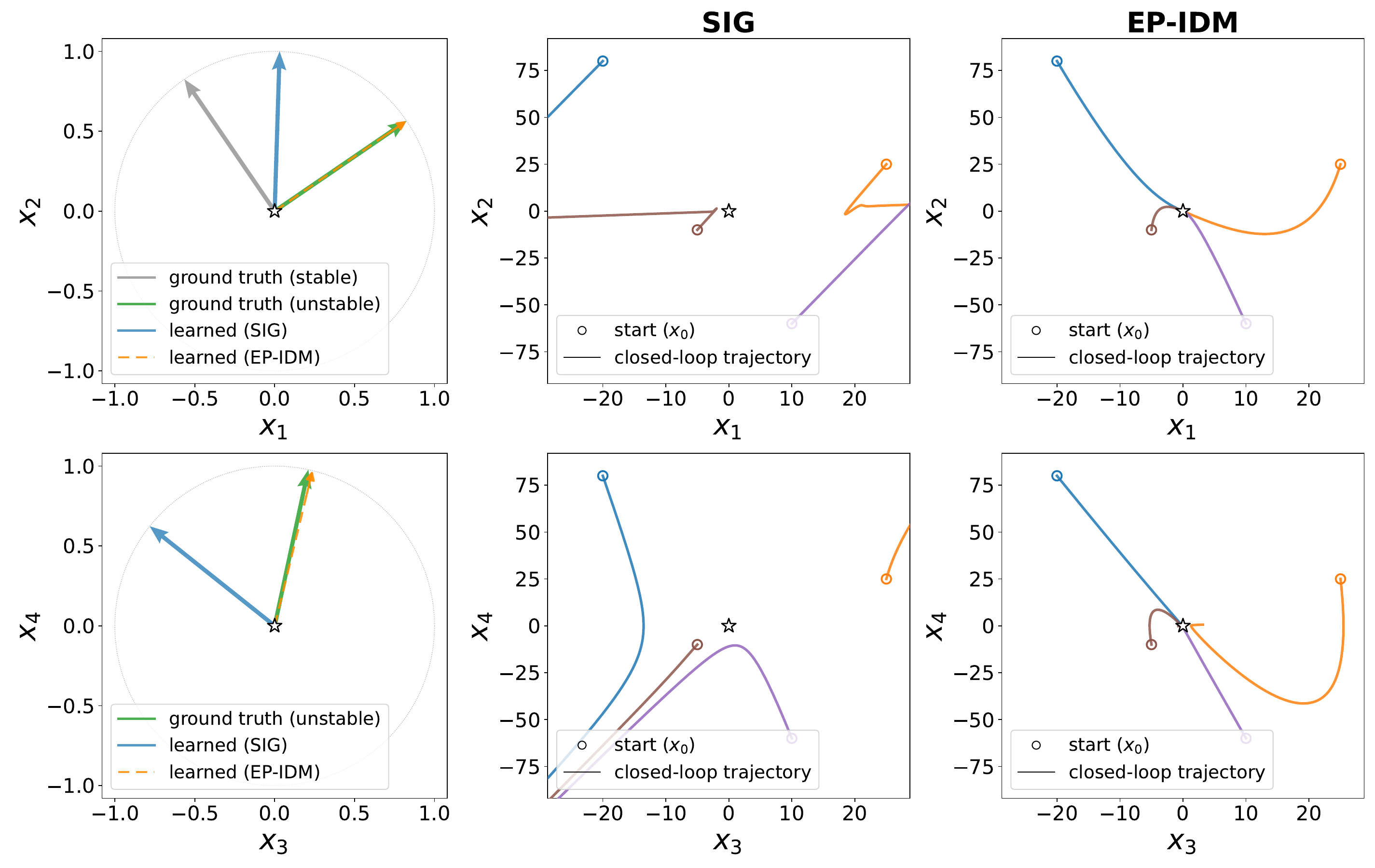}
\vspace{-0.4cm}
\caption{\small \textbf{SIG vs.\ EP-IDM across two examples.} Each row compares a
representation trained with SIGReg (SIG) against one trained with the
endpoint inverse-dynamics (action-reconstruction) loss (EP-IDM).
\textbf{Top row: synthetic open-loop unstable LDS.}
\textbf{Bottom row: linearized CartPole.}
\textbf{Left:} true unstable ($v_u$, green) and stable ($v_s$, grey)
eigenvectors versus the learned predictor's dominant eigenvector, for SIG
(blue) and EP-IDM (dashed orange). \textbf{Center:} closed-loop trajectories
under the latent LQR controller from four initial states (circles) for the
model trained with SIGReg. \textbf{Right:} closed-loop trajectories under
the latent LQR controller for the model trained with the inverse dynamics
loss (EP-IDM). For the bottom row $x_3$, and $x_4$ corresponds to the pole angle and angular velocity, respectively. The latent dimension is one for the example in the top row and three for the example in the bottom row.
} 
\label{fig:Examples_main}. 
\vspace{-0.5cm}
\end{figure}

We illustrate Lemma~\ref{lemma:collapse_of_unstable_modes} with a synthetic open-loop unstable system (top row of Figure~\ref{fig:Examples_main}) and a linearized CartPole system (bottom row). Appendix~\ref{appendix:example_details} provides details and an additional open-loop stable example. The left panels compare the learned unstable eigenvector, mapped back to the state space, with the ground-truth unstable eigenvector. The center panels show closed-loop trajectories for the physical systems under the latent linear quadratic regulator (LQR) trained with prediction and SIGReg through Eq.~(\ref{optimization:ms_SIGreg}). In both systems, SIG learns a substantially misaligned unstable direction, and the resulting controller fails to stabilize the physical system.\footnote{We use SIGReg to denote the sketched isotropic Gaussian regularization objective~\citep{balestriero2025lejepa}, and $\mathrm{SIG}$ to denote a model trained using this regularizer. When the
distinction is clear from context, we use the two terms interchangeably.}

\section{Preserving Unstable Directions Through Inverse Dynamics}
\label{sec:idm}

To encourage the learned representation to preserve unstable modes, we augment the world-model training objective with an inverse dynamics loss (IDM). In particular, our endpoint inverse dynamics (EP-IDM) loss reconstructs the entire action sequence over the prediction horizon from the initial and final latent states (see Fig.~\ref{fig:architecture} for the complete architecture).

\subsection{Endpoint Inverse Dynamics}

We train an inverse dynamics model\footnote{We augment $\theta$ with the trainable parameters of the action decoder.} $D_\theta: \mathcal{Z} \times \mathcal{Z} \to \mathcal{A}^H$ to reconstruct the action sequence from the initial and final latent states $z_{t_\tau}$ and $z_{t_\tau+H}$\footnote{We trained two other variants of this inverse dynamics loss that are detailed in Appendix~\ref{appendix:additinal_action_reconstruction}.}:
\begin{equation}
\mathcal{L}_{\mathrm{EP}\text{-}\mathrm{IDM}}(\theta)
= \mathbb{E}_{\tau \sim \mathcal{D}}\left[\frac{1}{H}\sum_{t=0}^{H-1}
\big\lVert \big[D_\theta(z_{t_\tau},\, z_{t_\tau+ H})\big]_t - a_{t_\tau+t} \big\rVert_2^2\right],
\end{equation}
where $[D_\theta(z_{t_\tau},\, z_{t_\tau+H})\big]_t$ denotes the predicted action at timestep $t$.
We jointly optimize the encoders, predictor, and inverse dynamics model solving the optimization problem in   (\ref{optimization:ms_SIGreg}) with $\mathcal{L}_{\mathrm{EP}\text{-}\mathrm{IDM}}(\theta)$ in lieu of $\mathcal{L}_{\mathrm{reg}}(\theta)$.

\subsection{Preservation of Reachable Directions}

We show that exact endpoint action reconstruction requires the encoder to preserve every direction reachable within $H$ steps.
For the linear analysis, we consider an affine inverse dynamics model
\begin{equation}
D_\theta(z_0,z_H) = W_H[z_0\;z_H] + b_H \text{ with } W_H \in \mathbb R^{mH \times 2d} \text{ and } b_H \in \mathbb R^{mH}.
\end{equation}
Let $\mathcal C_H=\begin{bmatrix}A^{H-1}B&A^{H-2}B&\cdots&B\end{bmatrix}$ denote the finite-horizon controllability matrix and let $\mathcal R_H=\operatorname{range}(\mathcal C_H)$ be the $H$-step reachable subspace.

We require the training data to contain sufficiently rich action excitation. For a fixed initial state $x_{t_\tau}$, let
$\operatorname{supp}(\mathbf a_{\tau,H}\mid x_{t_\tau})$ denote the support of the conditional distribution of the stacked action sequence $\mathbf a_{\tau,H}:=\begin{bmatrix}
a_{t_\tau}^\top &\cdots &a_{t_\tau+H-1}^\top
\end{bmatrix}^{\top}\in\mathbb R^{mH}.$

\begin{theorem}
\label{theorem:condition_achieved_action_reconstruction}
Consider the linear system in Eq.~\eqref{system_LDS} and suppose that, for almost every initial
state $x_{t_\tau}$, the conditional support
$\operatorname{supp}(\mathbf a_{\tau,H}\mid x_{t_\tau})$ contains a
nonempty open subset of $\mathbb R^{mH}$. If an affine inverse dynamics
model achieves
$\mathcal L_{\mathrm{EP}\text{-}\mathrm{IDM}}=0$, then
\begin{align}
\label{eq:injective_reachable_subspace}
\ker(F)\cap\mathcal R_H=\{0\}.
\end{align}
Equivalently, the encoder is injective on the $H$-step reachable
subspace: if $v\in\mathcal R_H$ and $Fv=0$, then $v=0$. Thus, no state
direction that can be generated by an action sequence within $H$ steps
is discarded by the encoder. In particular, if $\Phi_{\geq1}\subseteq\mathcal R_H$, then $\ker(F)\cap\Phi_{\geq1}=\{0\}.$
Exact reconstruction $\mathcal L_{\mathrm{EP}\text{-}\mathrm{IDM}} = 0$ holds only if $\operatorname{rank}(F\mathcal C_H)=mH
\leq\min\{d,n\}.$
\end{theorem}

The proof of Theorem \ref{theorem:condition_achieved_action_reconstruction} is provided in~Appendix~\ref{appendix:proof_action_reconstruction}. We note that under Assumption \ref{assumption:stabilizable_observable} and a sufficiently large horizon $H$,  we have that $\Phi_{\geq1}\subseteq\mathcal R_H$. Therefore, the theorem above ensures that all unstable and marginally unstable modes remain observable through the representation, satisfying the detectability condition in Lemma~\ref{lemma:conditions_for_stabilization}.

Theorem~\ref{theorem:condition_achieved_action_reconstruction} describes the idealized exact-reconstruction regime. We note that empirically, we use $\mathcal{L}_{\mathrm{EP}\text{-}\mathrm{IDM}}$ as a soft inductive bias and do not require it to be zero. Indeed, when $mH>d$, exact recovery of an arbitrary action sequence is not feasible. Nevertheless, minimizing the loss can still encourage the representation to preserve the action-induced directions required for control (Section \ref{sec:numerics}).

As illustrated in Fig.~\ref{fig:Examples_main}, training the latent world model with the endpoint inverse dynamics loss (EP-IDM) preserves the unstable direction of the physical system in the latent space. The learned latent dynamics capture the unstable mode that must be corrected through feedback, enabling latent stabilization, as shown in the right panels of Fig.~\ref{fig:Examples_main}. 
In addition, our result in Theorem \ref{theorem:condition_achieved_action_reconstruction} motivates the question we address next: \emph{Which controllable directions preserved by the action reconstruction are the most identifiable ones?}

\subsection{Reachability and Unstable Dynamics}

We consider the finite-horizon controllability Gramian given by $\Pi_H=\sum_{k=0}^{H-1}A^kBB^\top(A^k)^\top.$

For simplicity, we exclude eigenvalues on the unit circle, so $\Phi_{\geq1}$ coincides with the strictly unstable subspace $\Phi$. Let $\|\sin\Theta(\hat{\Phi},\Phi)\|_2$ denote their subspace distance, as defined in Definition~\ref{definition:principal_angle_distance}.

\begin{theorem}
\label{theorem:unstable_subspace_recovery}
Under the mild assumptions given in Appendix \ref{appendix:proof_unstable_subspace_recovery}, let $\hat{\Phi}$ be the span of the top
$r=\dim(\Phi)$ eigenvectors of $\Pi_H$. Then, we have
\begin{equation}
\|\sin\Theta(\hat{\Phi},\Phi)\|_2 \to 0
\qquad \text{as } H\to\infty.
\end{equation}
\end{theorem}
The convergence rates and proof are provided in
Appendix~\ref{appendix:proof_unstable_subspace_recovery}. 

We note that although $\operatorname{range}(\Pi_H)=\operatorname{span}
(B,AB,\ldots,A^{H-1}B)$ does not acquire new directions once $H$ exceeds
the state dimension, the relative weighting of these directions
continues to change with $H$. Theorem~\ref{theorem:unstable_subspace_recovery}
demonstrates that the effects of unstable modes grow exponentially and eventually
dominate the stable ones, making the leading eigenspace of $\Pi_H$ to
converge to the unstable subspace $\Phi$. 

Moreover, Assumption~\ref{assumption:unstable_reachability} in Theorem~\ref{theorem:unstable_subspace_recovery} requires
$\lambda_{\min}(\Pi_{u,H})\geq c_u\alpha^{2H}$, where $\Pi_{u,H}$
denotes the unstable block of the finite-horizon controllability
Gramian. We emphasize that this condition strengthens controllability by
requiring every unit unstable direction to exhibit an action-induced
response growing at least as $\alpha^{2H}$. It follows when
$(A_u,B_u)$ is controllable and
$\sigma_{\min}(A_u)\geq\alpha>1$, as discussed in
Appendix~\ref{appendix:proof_unstable_subspace_recovery}. For $\Delta\neq0$, Assumption \ref{assumption:unstable_reachability} further requires that the indirect input response transmitted through the stable coordinates and
the coupling $\Delta$ does not cancel the direct unstable response
generated through $B_u$. Therefore, 
\begin{align}
\lambda_{\min}(\Pi_{u,H}) =
\min_{\|v\|_2=1}v^\top\Pi_{u,H}v
\geq c_u\alpha^{2H},
\end{align}
or equivalently,
\begin{align}
v^\top\Pi_{u,H}v \geq c_u\alpha^{2H}
\qquad \text{for every } \|v\|_2=1
\end{align}
requires that the finite-horizon controllability energy grows at the same exponential scale dictated by the unstable modes, which strengthens the condition of all unstable directions being reachable.  

Therefore, taken together, Theorems \ref{theorem:condition_achieved_action_reconstruction} and \ref{theorem:unstable_subspace_recovery} demonstrate that action reconstruction, in particular EP-IDM, prevents the encoder from discarding finite-horizon controllable directions. The controllability Gramian further explains why the dominant action-induced controllable directions increasingly align with the unstable subspace as $H$ grows, i.e., why the unstable directions are the most identifiable ones. 

We also note that, for simplicity, Theorem~\ref{theorem:unstable_subspace_recovery} excludes eigenvalues on the unit circle. Therefore, its instantiation to the
linearized CartPole example in Figure~\ref{fig:Examples_main}
concerns only the strictly unstable mode depicted in the left panel.

\section{Experiments}
\label{sec:numerics}

We now validate\footnote{Code and checkpoints to reproduce the results can be found at \url{https://jepa-control.github.io/}.} our approach on nonlinear visual-control tasks in MuJoCo \citep{todorov2012mujoco}. Our experiments address three questions: (i) whether standard predictive JEPA objectives (i.e., solving Eq. (\ref{optimization:ms_SIGreg}) with SIGReg as a regularizer) preserve the information required for feedback stabilization, (ii) whether EP-IDM preserves that information, and (iii) whether the resulting representations remain useful for stable systems. We consider CartPole, Walker2D, and PointMaze systems (See Fig.~\ref{fig:tasks_frames} for the visualizations of successful trials for CartPole, Walker2D, and PointMaze). Additional experiments and ablations are provided in Appendix~\ref{appendix:additional_experimental_validation}.

\begin{figure}[t]
\centering
\includegraphics[width=\textwidth]{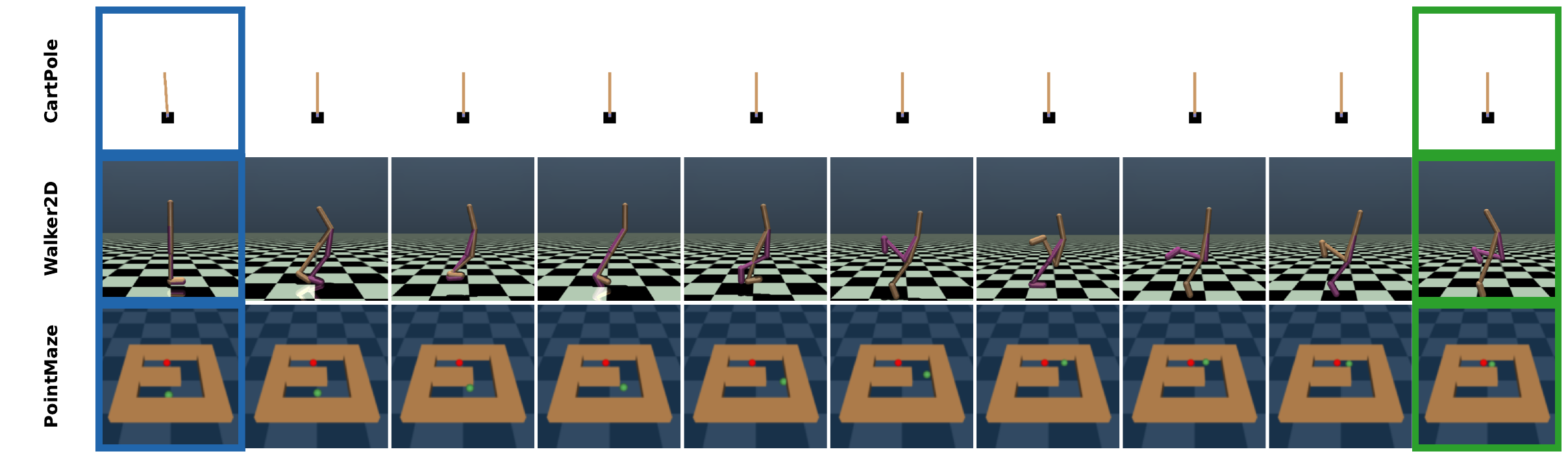}
\caption{\small \textbf{CartPole}, \textbf{Walker2D}, and \textbf{PointMaze} visual-control tasks used in our evaluation. Each row depicts ten frames from a successful trial (initial state: blue border, final state: green border). CartPole: a continuous-action inverted-pendulum balancing task controlled via LQR. Walker2D: a bipedal robot locomotion task controlled via iCEM \citep{pinneri2021sample}. PointMaze: a 2-D point-mass robot navigation task planned via CEM.}
\label{fig:tasks_frames}
\end{figure}

\begin{wrapfigure}{r}{0.65\textwidth}
\vspace{-1cm}
\centering
\includegraphics[width=0.65\textwidth]{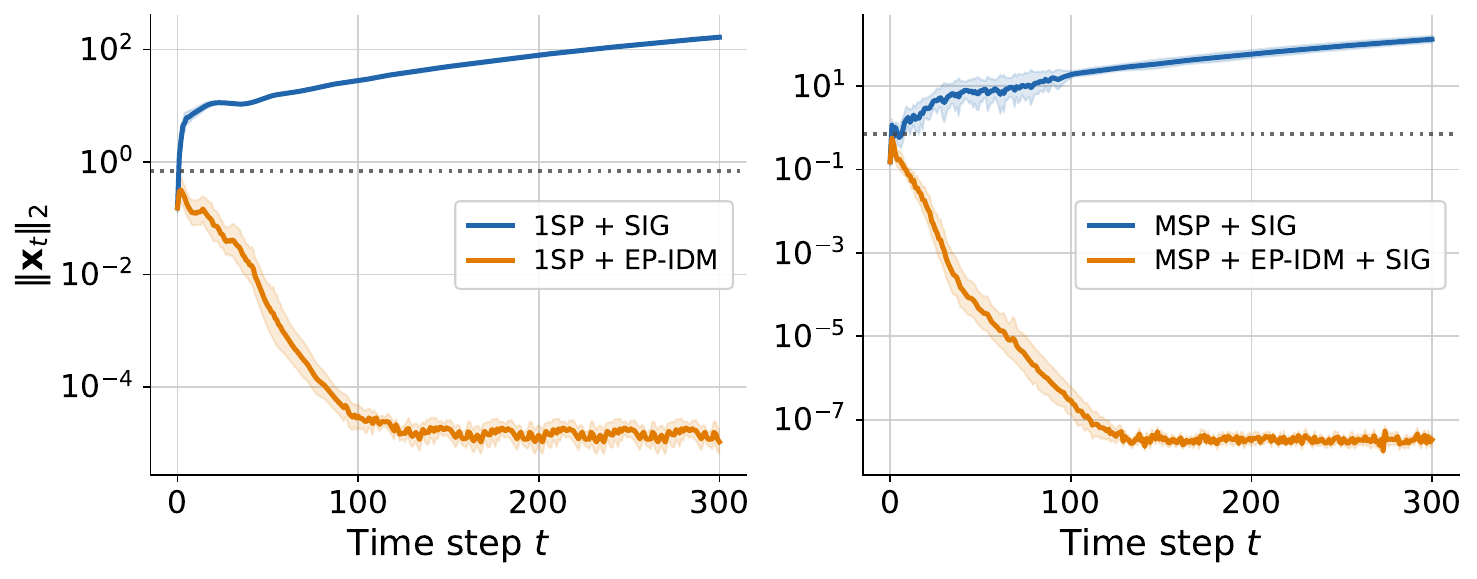}
\caption{\small \textbf{CartPole} state norm $\|x_t\|_2$ under the designed latent LQR controller over $300$ steps. 
}
\vspace{-0.2cm}
\label{fig:lqr_from_scratch}
\end{wrapfigure}

\subsection{Setup}

We consider both encoders trained from scratch and frozen pretrained encoders (DINOv2 \citep{oquab2023dinov2} and iBOT \citep{zhou2021ibot}, with results reported in Appendix \ref{appendix:pretrained}), for which we train a projector on top. Our training objectives combine either one-step prediction (1SP) or autoregressive multi-step prediction (MSP) with either SIGReg (SIG) or an inverse-dynamics objective (IDM).

Our primary objective is endpoint action reconstruction (EP-IDM), for which the decoder reconstructs the complete action sequence from the first and last latent states. We additionally evaluate one-step action reconstruction (IDM), which reconstructs each action from consecutive latent states, and multi-step action reconstruction (MS-IDM), which reconstructs the complete action sequence from the corresponding latent trajectory. Further details are provided in Appendix~\ref{appendix:additinal_action_reconstruction}. Architecture, optimization, and dataset details are provided in Appendix~\ref{appendix:experimental_details}.
\vspace{0.2cm}

\noindent \textbf{Systems.} CartPole requires stabilizing an open-loop unstable upright equilibrium, where small angular perturbations can make the pole fall. Walker2D is also open-loop unstable because maintaining a walking gait requires continuous feedback. PointMaze instead evaluates goal reaching in a U-shaped maze with open-loop stable point-mass dynamics.
\vspace{0.2cm}

\noindent \textbf{Controllers and planners.} We evaluate three uses of the learned world model. Latent LQR linearizes the predictor at the equilibrium and provides our most direct diagnostic: a collapsed unstable mode cannot be identified or corrected by the resulting feedback controller. CEM searches over sampled open-loop rollouts with repeated replanning, whereas GBP differentiates the terminal latent cost through the predicted rollout, both planners are implemented as receding-horizon MPC and periodically replan. Unlike LQR, neither requires the local linearization to capture the unstable modes. We report success over ten trials and provide all hyperparameters in Appendix~\ref{appendix:experimental_details}.\vspace{0.2cm}

\noindent\textbf{Metrics.} We report success rate (SR). For CartPole, success requires $\|x_T\|_2\leq0.7$ after the $300$-step rollout, and the mean fraction of steps (MFS) satisfying $\|x_t\|_2\leq0.7$ measures how consistently the controller remains near the upright equilibrium. For PointMaze, success requires $\|p_T-p_{\mathrm g}\|_2\leq0.5$, where $p_T$ and $p_{\mathrm g}$ are the final and goal positions. For Walker2D, we evaluate whether the learned latent dynamics preserve the limit cycle induced by the locomotion policy used to collect the training data. We also do planning through the learned latent dynamics for Walker2D with iCEM \citep{pinneri2021sample}, where success requires keeping forward locomotion without falling over the evaluation horizon. Additional details are provided in Appendix~\ref{appendix:experimental_details}.

\begin{figure}[t]
\centering
\includegraphics[width=1\textwidth]{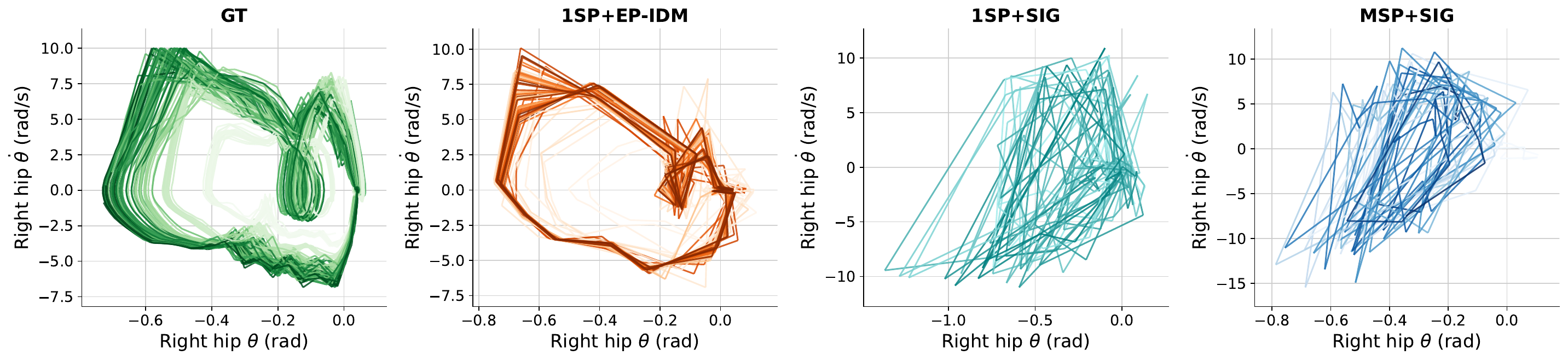}
\caption{\small \textbf{Walker2D.} Phase portraits of right-hip angle vs. angular velocity for GT dynamics, 1SP+EP-IDM, 1SP+SIG, and MSP+SIG. Each trajectory is a $500$-step closed-loop rollout: GT uses real MuJoCo physics, model panels roll out in latent space under an MLP-decoded state, with a SAC~\citep{huang2022cleanrl,haarnoja2018soft} policy acting on the decoded observation.}
\label{fig:limit_cycle_walker2d}
\vspace{-0.1cm}
\end{figure}

\subsection{CartPole: Stabilization Around an Unstable Equilibrium}

Table~\ref{tab:from-scratch-sigreg-idm} shows the effect of action reconstruction on stabilization performance when the encoder and predictor are trained from scratch. Prediction with SIGReg yields zero latent-LQR success and low MFS, whereas adding EP-IDM achieves $100\%$ success. This confirms that standard prediction and anti-collapse regularization can discard controllable unstable directions, while action reconstruction encourages the representation to preserve the finite-horizon reachable directions containing them.

\begin{table}[t]
\centering
\caption{\small\textbf{CartPole} stabilization success rate (\%). 
}
\label{tab:from-scratch-sigreg-idm}
\small
\begin{tabular}{lcccc}
\toprule
Model & LQR (SR) & LQR (MFS) & CEM (SR) & GBP (SR) \\
\midrule
Ground truth (GT) & \textbf{100\%} & \textbf{1.000} & \textbf{100\%} & \textbf{100\%} \\
\midrule
1SP + SIG & 0\% & 0.202 & \textbf{100\%} & 90\% \\
MSP + SIG & 0\% & 0.247 & \textbf{100\%} & \textbf{100\%} \\
\midrule
1SP + EP-IDM & \textbf{100\%} & 0.999 & \textbf{100\%} & \textbf{100\%} \\
MSP + EP-IDM + SIG & \textbf{100\%} & 0.996 & \textbf{100\%} & \textbf{100\%} \\
\bottomrule
\end{tabular}
\end{table}

CEM and GBP succeed even with SIGReg models (Table~\ref{tab:from-scratch-sigreg-idm}). While LQR relies on a local linearization and directly verifies whether the learned model supports local feedback, CEM and GBP optimize finite action sequences through nonlinear predictor rollouts. Their success indicates that SIGReg models retain enough information for nonlinear replanning, even when their local latent linearizations do not yield stabilizing LQR controllers.
Figure~\ref{fig:lqr_from_scratch} makes the failure of the standard JEPA recipe more explicit: under 1SP+SIG and MSP+SIG, the state norm grows without bound, whereas EP-IDM enables the latent LQR controller to drive the state toward the equilibrium and keep it there.

\subsection{Walker2D: Preserving a Walking Gait}

For Walker2D, we evaluate whether action reconstruction preserves control-relevant dynamics beyond a single unstable equilibrium. A walking gait forms a limit cycle whose deviations must be continuously corrected. Figure~\ref{fig:limit_cycle_walker2d} compares the right-hip phase portrait from the ground-truth MuJoCo simulator with latent rollouts from 1SP+EP-IDM, MSP+SIG, and 1SP+SIG. EP-IDM preserves the cyclic structure and range of joint angles and angular velocities, whereas 1SP+SIG and MSP+SIG fail to capture the limit cycle induced by the locomotion policy. Hence, we confirm empirically that action reconstruction also preserves control-relevant information associated with a locomotion orbit which goes beyond the scope of preserving unstable equilibria.

\begin{figure}[t]
\centering
\includegraphics[width=1\textwidth]{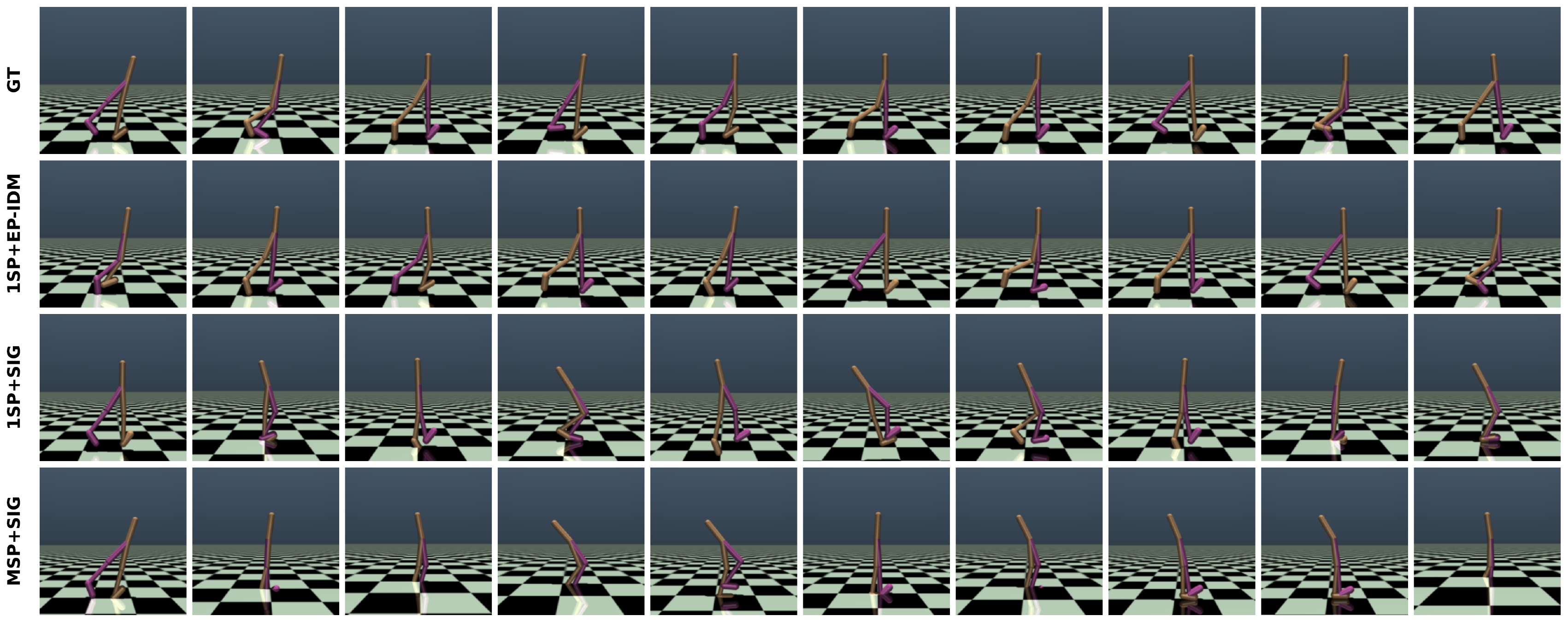}

\caption{\small Trajectory frames for \textbf{Walker2D}. Each row shows ten frames uniformly sampled from a $500$-step closed-loop rollout under the SAC policy \citep{haarnoja2018soft}, decoded from the trained model's latent state via an MLP decoder. } 
\label{fig:walker2d_limitcycle}
\end{figure}

Figure ~\ref{fig:walker2d_limitcycle} complements Fig.~\ref{fig:limit_cycle_walker2d} where we visualize the decoded states directly in pixel space. GT dynamics (top row) produce a periodic bipedal gait throughout
the rollout. 1SP+EP-IDM (second row) also keeps coherent walking motion on the episode, confirming that the model captures the structure of the walking limit cycle. 1SP+SIG (third row) and MSP+SIG (bottom row) both fail to keep a physically
plausible locomotion: the decoded states are erratic and unnatural over the course of the rollout, which is not consistent with a periodic bipedal walking gait as also depicted in Fig.~\ref{fig:limit_cycle_walker2d}.

\begin{table}[t]
\centering
\caption{\small\textbf{Walker2D} locomotion under latent iCEM \citep{pinneri2021sample}.
All metrics are averaged over ten trials. Final height is the torso height (m) at the end of the episode.}
\label{tab:walker2d-icem}
\small
\begin{tabular}{lccc}
\toprule
Model & Avg Velocity (m/s) & Displacement (m) & Final Height (m) \\
\midrule
Ground truth (GT) & \textbf{3.61} & \textbf{14.43} & \textbf{1.21} \\
\midrule
1SP + SIG    & 0.08 & 0.06 & 0.98 \\
MSP + SIG    & 0.46 & 0.66 & 0.85 \\
\midrule
1SP + EP-IDM & 2.91 & 9.15 & 0.92 \\
\bottomrule
\end{tabular}
\end{table}

\begin{figure}[t]
\centering
\includegraphics[width=1\textwidth]{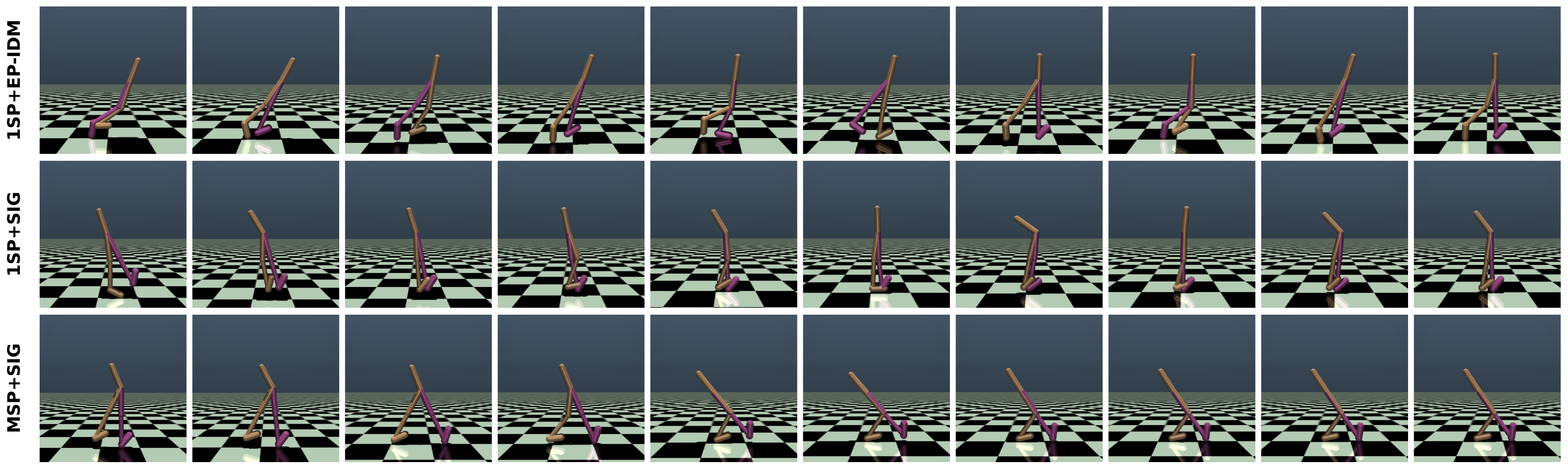}
\caption{\small  Planning trajectory frames for \textbf{Walker2D}. Each row shows ten frames uniformly sampled from a $500$-step closed-loop rollout under latent iCEM planning for the trained models. We select the best trial, with respect to the forward displacement, across the ten trials for the trained models.} 
\label{fig:walker2d_planning}
\end{figure}

Table~\ref{tab:walker2d-icem} and Fig.~\ref{fig:walker2d_planning} evaluate the closed-loop control performance of the learned models under the latent iCEM \citep{pinneri2021sample} across ten trials. The ground-truth planner achieves $3.61$ m/s and $14.43$ m of forward displacement. Among the learned models, 1SP+EP-IDM reaches $2.91$ m/s and $9.15$ m, substantially closer to the ground truth behavior than the models trained with SIGReg, which produce near-zero forward displacement. The planning trajectories in Fig.~\ref{fig:walker2d_planning} confirm that, with 1SP+EP-IDM, we are able to keep a coherent bipedal gait throughout the episode, while 1SP+SIG and MSP+SIG cannot keep such walking gait.

\subsection{Diagnostics}

\begin{figure}[t]
\centering
\includegraphics[width=1\textwidth]{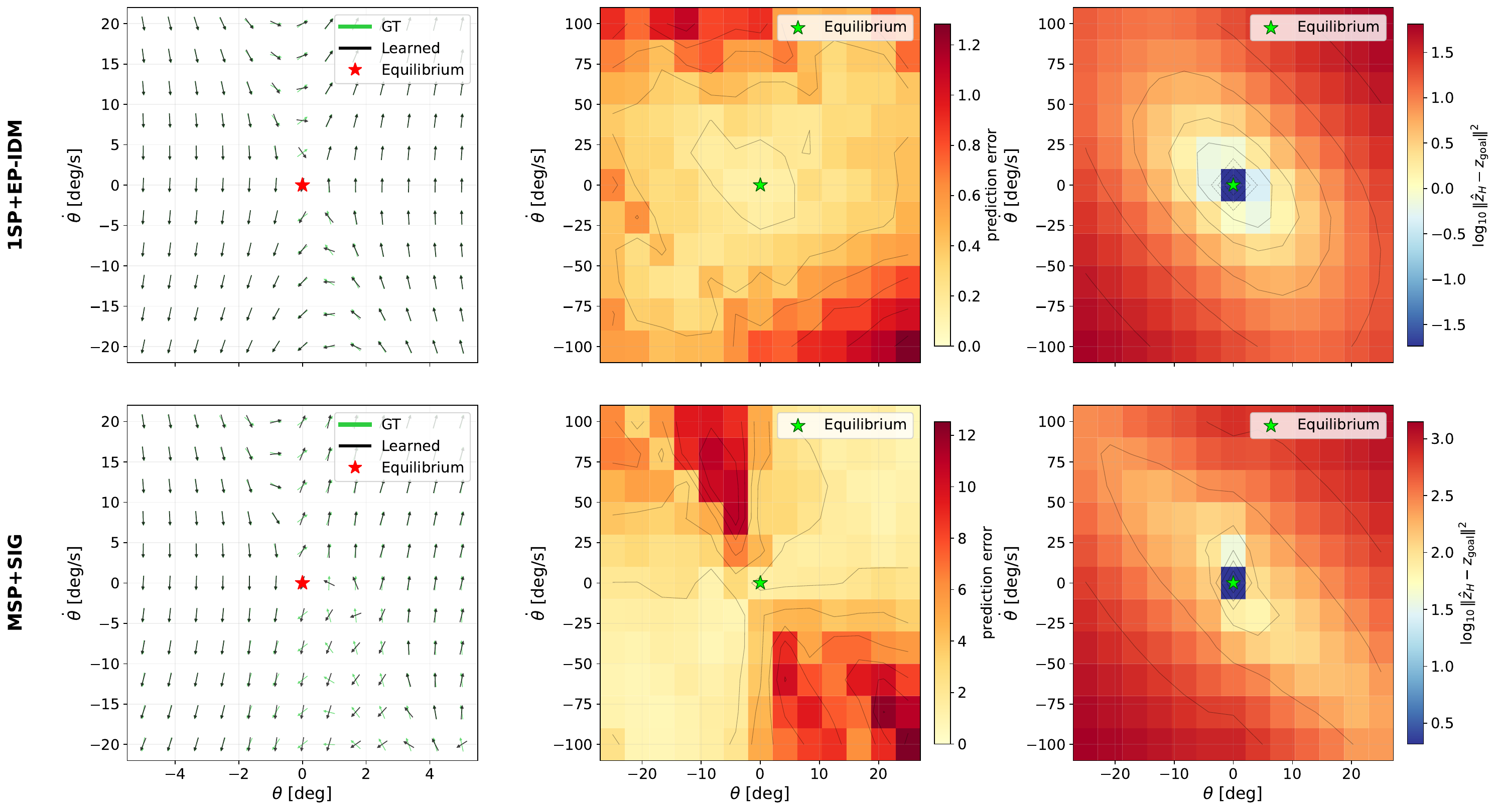}
\caption{\small\textbf{CartPole} diagnostics.
    \textbf{Left:} Phase portraits of ground-truth and learned vector fields.
    \textbf{Center:} $H=3$ open-loop prediction error across
    $(\theta,\dot{\theta})$.
    \textbf{Right:} Zero-action planning cost.}
\label{fig:landscape_MSP_SIG}
\end{figure}

\begin{table}[t]
\centering
\begin{threeparttable}
\caption{\small \textbf{PointMaze} navigation success rate (\%). }
\label{tab:pointmaze_context_one}
\setlength{\tabcolsep}{5.5pt}
\begin{tabular}{@{}llccc@{}}
\toprule
Encoder & Objective & CEM (SR) & LQR (SR) & GBP (SR) \\
\midrule
\multicolumn{2}{l}{GT}
    & \textbf{100\%}
    & \textbf{80\%}
    & 90.0\%\tnote{*} \\
\midrule
From scratch
    & 1SP + IDM
    & \textbf{100\%} & 70\% & \textbf{100\%} \\
    & MSP + MS-IDM
    & \textbf{100\%} & \textbf{80\%} & 90.0\% \\
    & MSP + EP-IDM
    & 90.0\% & \textbf{80\%} & 90.0\% \\
    & 1SP + EP-IDM
    & \textbf{100\%} & \textbf{80\%} & 80.0\% \\
    & MSP + SIG
    & 90.0\% & 60\% & \textbf{100\%} \\
    & 1SP + IDM + MS-IDM
    & 90.0\% & \textbf{80\%} & 90.0\% \\
\midrule
DINOv2
    & MSP
    & 50.0\% & 20\% & 80\% \\
\midrule
iBOT
    & MSP
    & \textbf{100\%} & \textbf{80\%} & 80\% \\
\midrule
\multicolumn{2}{l}{DINO-WM}
    & \textbf{80\%} & \textbf{80\%} & 60\% \\
\bottomrule
\end{tabular}
\vspace{-0.05cm}
\begin{tablenotes}[flushleft]
\footnotesize
\item[*] Because MuJoCo is non-differentiable, we estimate gradients with SPSA using two rollouts per optimization step~\citep{spall1998overview}.
\end{tablenotes}
\end{threeparttable}
\end{table}

Figure~\ref{fig:landscape_MSP_SIG} compares the local dynamics learned by 1SP+EP-IDM and MSP+SIG. 1SP+EP-IDM reproduces the ground-truth vector field near the upright equilibrium with low three-step prediction error, whereas MSP+SIG captures only parts of the field and exhibits larger, anisotropic error. Both models produce a latent goal cost with a minimum near the equilibrium, which CEM and GBP can exploit through finite-horizon optimization, while LQR requires an accurate local linearization.

Figure~\ref{fig:stability} shows the empirical region of attraction (ROA) and Lyapunov difference under the latent LQR controllers. The controller obtained from 1SP+EP-IDM stabilizes almost every evaluated initial condition and yields $\Delta V=V(x_{t+1})-V(x_t)<0$ around the equilibrium, providing an empirical certificate of local closed-loop stability. In contrast, the MSP+SIG controller stabilizes only a small neighborhood around the equilibrium.

Although MSP+SIG provides a terminal cost suitable for nonlinear planning (Figure~\ref{fig:landscape_MSP_SIG}), its local representation and predictor do not yield a latent stabilizing controller. Action reconstruction instead learns local dynamics from which a stabilizing controller can be designed, consistent with Section~\ref{sec:idm}.

\begin{figure}[t]
\centering
\includegraphics[width=1\textwidth]{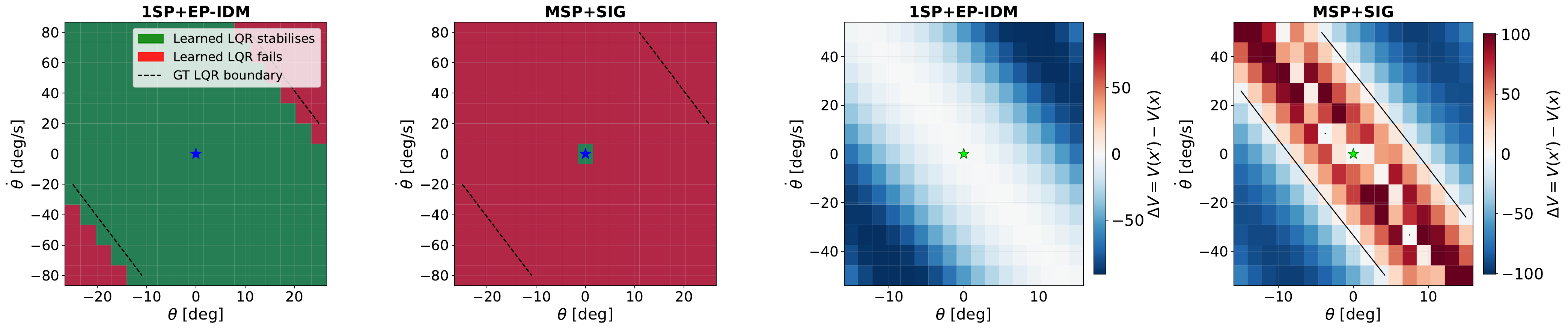}
\caption{\small Local stability on \textbf{CartPole}.
\textbf{First two panels:} Empirical region of attraction of the learned LQR controller when deployed in the ground-truth system. The dashed line
indicates the ground-truth LQR boundary.
\textbf{Last two panels:} Lyapunov decrease $\Delta V = V(x') - V(x)$ under the latent LQR controller.
}
\label{fig:stability}
\end{figure}

We evaluate whether action reconstruction remains effective for systems with no unstable modes using PointMaze (Table~\ref{tab:pointmaze_context_one}). 
We note that all action-reconstruction variants trained from scratch reach at least $90\%$ CEM success and at least $70\%$ LQR success.

\section{Conclusion and Future Work}

In this work, we demonstrated a fundamental limitation of next-step prediction combined with standard anti-collapse regularization (e.g., SIGReg) in JEPAs: these objectives do not necessarily preserve the control-theoretic information required to stabilize an open-loop unstable system. They can be minimized despite the encoder discarding controllable unstable modes. Therefore, the latent representation may not collapse on the training distribution, yet be insufficient for designing a stabilizing latent feedback controller. 

To overcome this limitation, we proposed endpoint inverse dynamics (EP-IDM), which reconstructs the action sequence from the initial and final latent representations. For linear systems, we proved that exact reconstruction preserves the finite-horizon reachable subspace and prevents the collapse of reachable unstable modes, while the dominant eigenspace of the controllability Gramian converges to the controllable unstable subspace as the prediction horizon grows. Empirically, EP-IDM increases latent-LQR success from $0\%$ to $100\%$ on nonlinear CartPole, preserves the Walker2D limit cycle and supports successful planning with $9.15$ m of average forward displacement over $500$ time steps, and achieves strong planning performance on PointMaze. 

Future work could combine our control-aware representation with latent-geometry objectives for efficient and robust planning, including bisimulation~\citep{toso2026learning} and temporal straightening~\citep{wang2026temporal}.

\section{Acknowledgments} 
Leonardo F. Toso and James Anderson thank Paul Lutkus and Professor Stephen Tu for insightful discussions during the early stages of this work. The authors also thank Professor Jean Ponce for his detailed comments and feedback on an earlier version of this manuscript. Leonardo F. Toso is funded by the Center for AI and Responsible Financial Innovation (CAIRFI) Fellowship and the Columbia Presidential Fellowship.  James Anderson is partially funded by NSF grants EECS 2144634 and CNS 2535097 and the Center of AI Technology (CAIT) in collaboration with Amazon. This work was also supported in part by AFOSR under grant FA95502310139.

\newpage

\bibliographystyle{apalike}
\bibliography{references}

\newpage 
\appendix
\parttoc 
\section{Appendix}

This appendix is organized as follows.
Appendix~\ref{appendix:additional_experimental_validation} supplements
Section~\ref{sec:numerics} with additional experimental results and
diagnostics. Appendix~\ref{appendix:pretrained} evaluates the trained world models
using frozen pretrained encoders.
Appendix~\ref{appendix:experimental_details} provides the model architectures, datasets, optimization parameters, controller implementations, and additional action-reconstruction objectives used in our experiments. Appendix~\ref{appendix:ablations} presents ablations on the action reconstruction, action-context length, and observation inputs. Appendix~\ref{appendix:example_details}
provides the construction and physical parameters of the linear examples from Section~\ref{sec:stabilization_conditions}. Finally,
Appendix~\ref{appendix:supporting_lemmas} states the supporting technical results, and Appendices~\ref{appendix:proof_conditions_for_stabilization}--%
\ref{appendix:proof_unstable_subspace_recovery} provide the proofs of
our theoretical guarantees.

\subsection{Additional Experimental Validation (Extending Section \ref{sec:numerics})}
\label{appendix:additional_experimental_validation}

We provide additional evidence and illustration  supporting the performance results in Section~\ref{sec:numerics}. We first examine representative CartPole and PointMaze trajectories and relate their closed-loop performance to the geometry and local dynamics learned by each model. We then evaluate long-horizon prediction along feedback trajectories to distinguish accurate rollout prediction from the preservation of the unstable directions required for stabilization.

\subsubsection{CartPole}
\label{appendix:CartPole_additional}

In Fig.~\ref{fig:CartPole_trajectories} we complement the success rates in Section~\ref{sec:numerics} with the sequences of frames in a randomly selected trial trajectory under the latent LQR controllers. The models trained with EP-IDM, both from scratch and with the frozen iBOT encoder, keep the pole upright throughout the rollout. In contrast, the prediction combined with SIGReg models, frozen iBOT without EP-IDM, and DINO-WM fail to stabilize the system. 

\begin{figure}[t]
\centering
\includegraphics[width=1\textwidth]{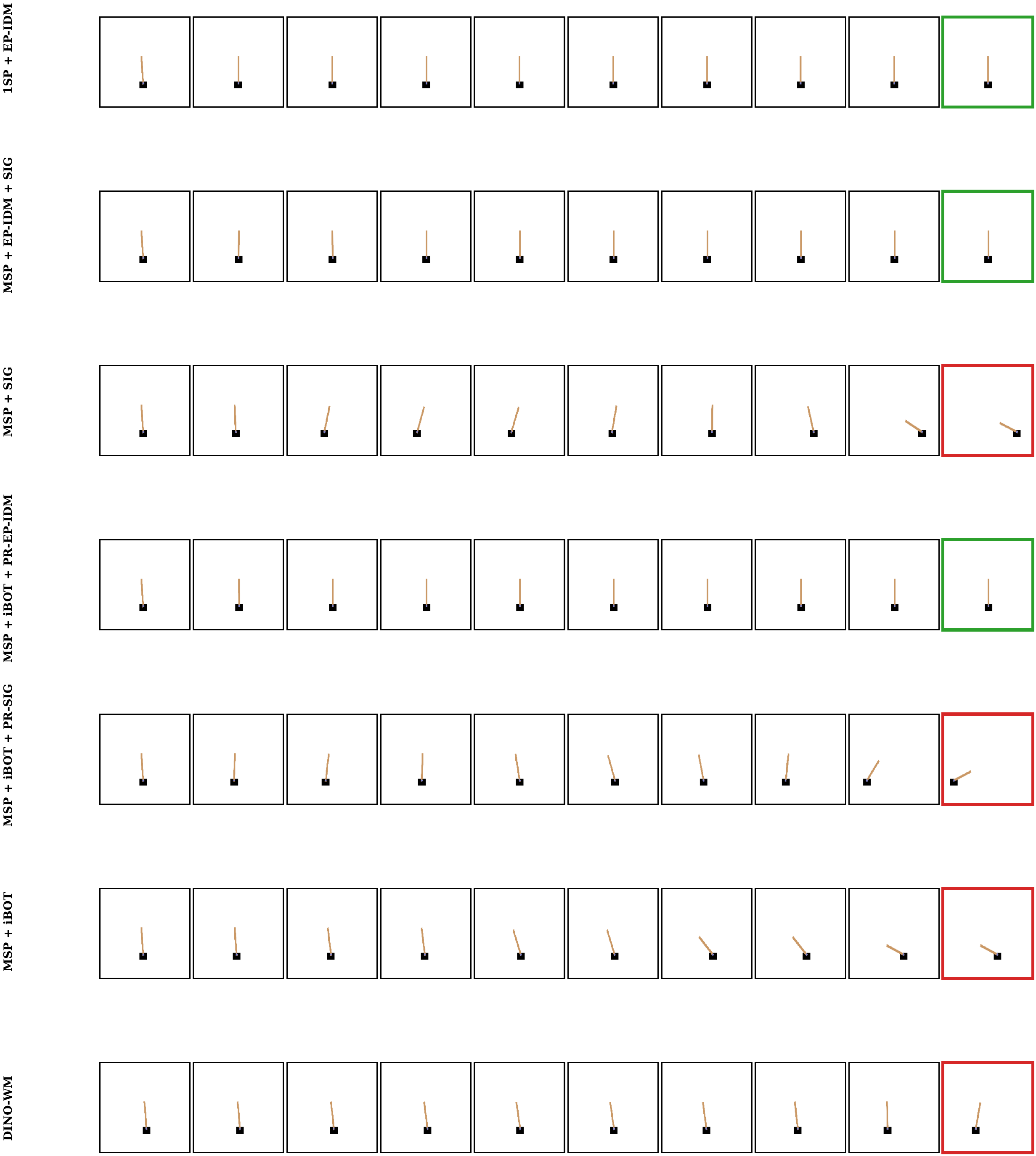}
\vspace{-0.5cm}
\caption{\small \textbf{CartPole} trajectory visualization for the trained models under the latent LQR controller.} \vspace{-0.7cm}
\label{fig:CartPole_trajectories}
\end{figure}

\begin{figure}[t]
\centering
\includegraphics[width=1\textwidth]{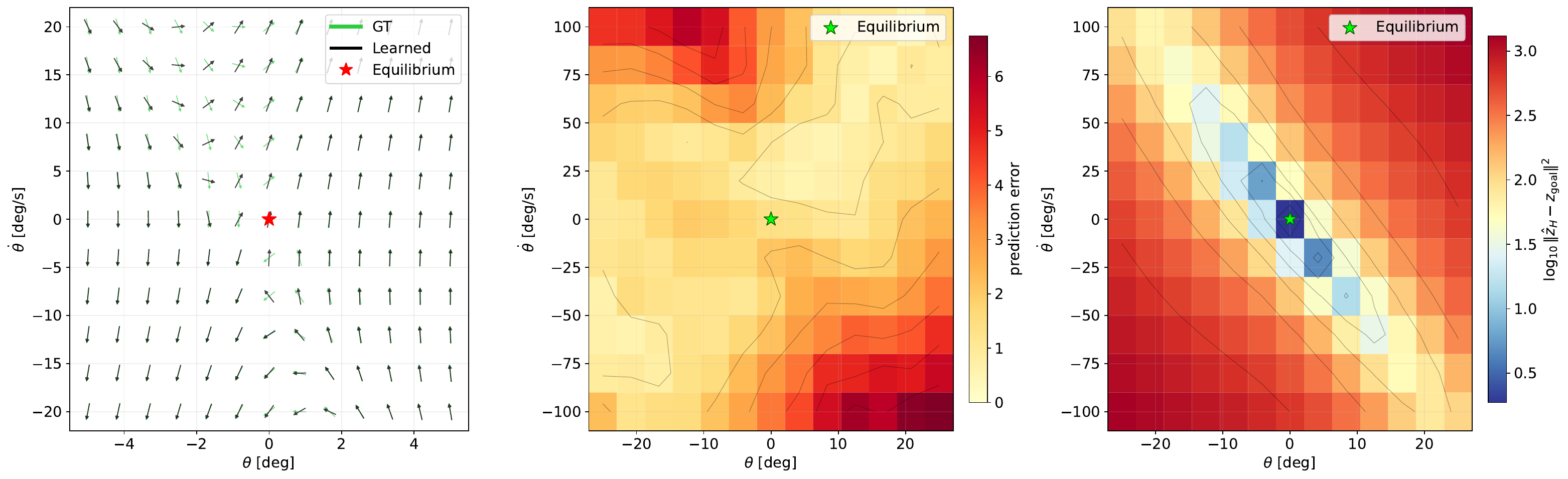}
\vspace{-0.5cm}
\caption{\small
MSP+EP-IDM+SIG model on \textbf{CartPole}.
\textbf{Left:}~Phase portrait of ground-truth vs. learned vector fields. \textbf{Center:}~$H=3$ open-loop prediction error. \textbf{Right:}~Zero-action planning cost.
} 
\label{fig:landscape_MSP_EP-IDM_SIG}
\end{figure}

\begin{figure}[t]
\centering
\includegraphics[width=1\textwidth]{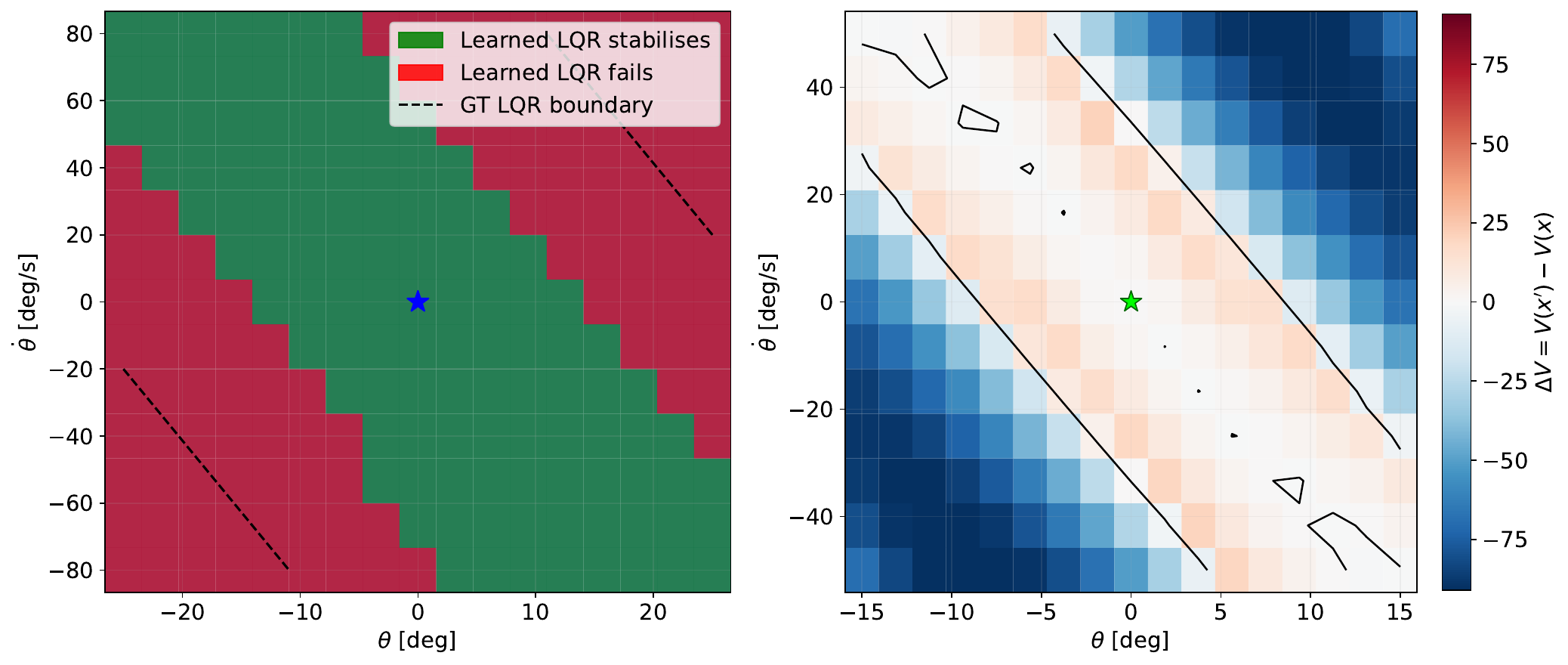}
\vspace{-0.5cm}
\caption{%
\small Local stability of MSP+EP-IDM+SIG on \textbf{CartPole}. \textbf{Left:} Empirical region of attraction (ROA) of the learned LQR policy in the GT environment. Dashed line: GT LQR boundary. \textbf{Right:} Lyapunov decrease $\Delta V = V(x') - V(x)$ under the learned policy using the GT DARE solution.}
\label{fig:stability_MSP_EP_IDM_SIG}
\end{figure}

\begin{figure}[t]
\centering
\includegraphics[width=1\textwidth]{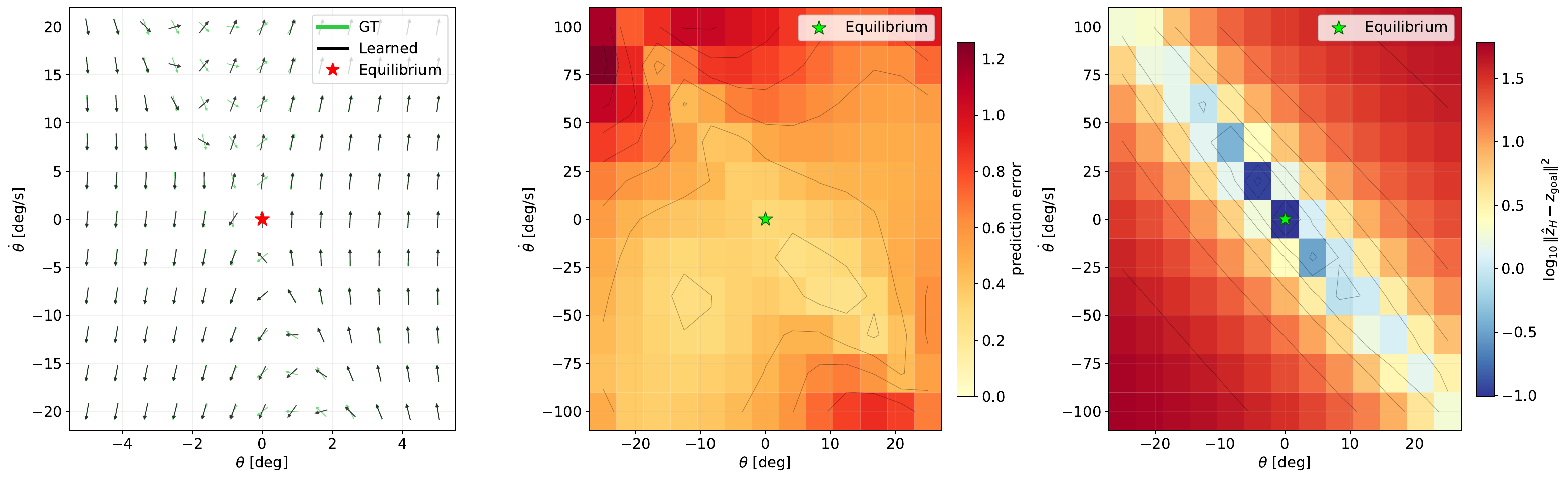}
\vspace{-0.5cm}
\caption{%
\small MSP+IBOT+PR-EP-IDM model on \textbf{CartPole}. Panel descriptions are in Fig.~ \ref{fig:landscape_MSP_EP-IDM_SIG}.
} 
\label{fig:landscape_MSP_IBOT_PR_EP_IDM}
\end{figure}

\begin{figure}[t]
\centering
\includegraphics[width=1\textwidth]{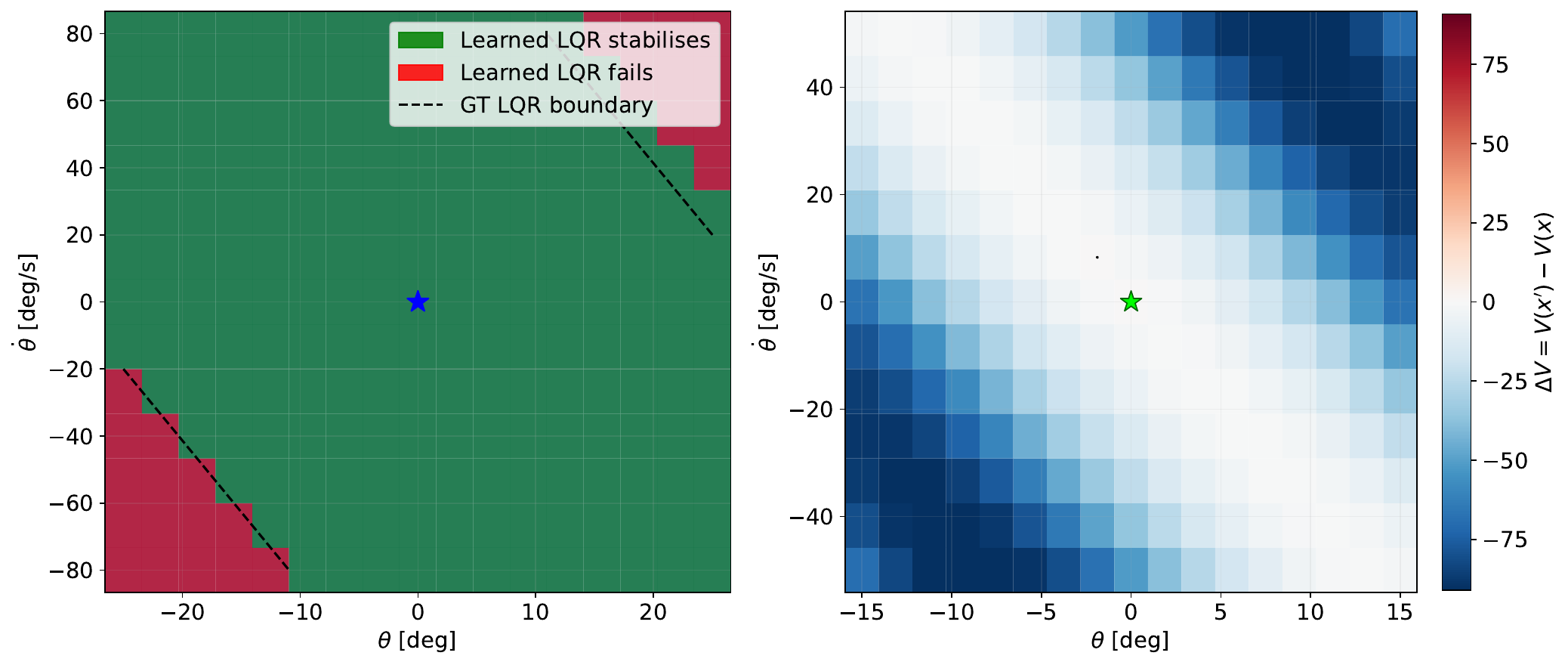}
\vspace{-0.5cm}
\caption{\small Local stability of MSP+IBOT+PR-EP-IDM on \textbf{CartPole}. Panel descriptions are the same as those in Fig.~ \ref{fig:stability_MSP_EP_IDM_SIG}.}
\label{fig:stability_MSP_IBOT_PR_EP_IDM}
\end{figure}

\begin{figure}[t]
\centering
\includegraphics[width=1\textwidth]{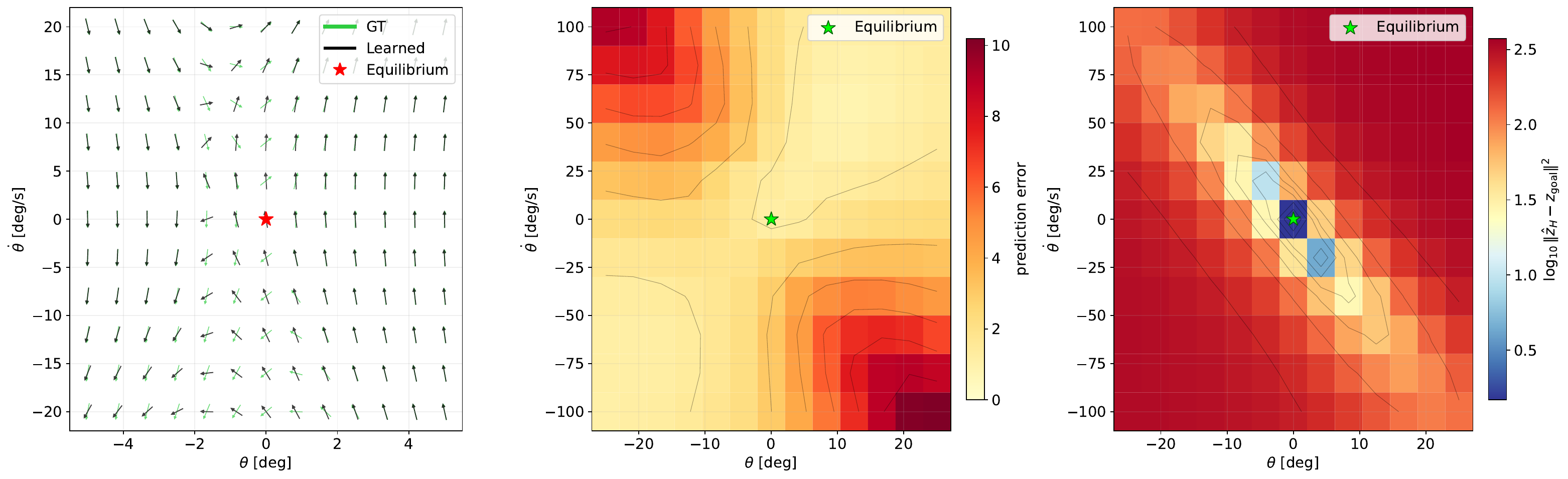}
\vspace{-0.5cm}
\caption{%
\small MSP+IBOT+PR-SIG model on \textbf{CartPole}.
Panel descriptions are the same as those in Fig.~ \ref{fig:landscape_MSP_EP-IDM_SIG}.
} 
\label{fig:landscape_MSP_IBOT_PR_SIG}
\end{figure}

\begin{figure}[t]
\centering
\includegraphics[width=1\textwidth]{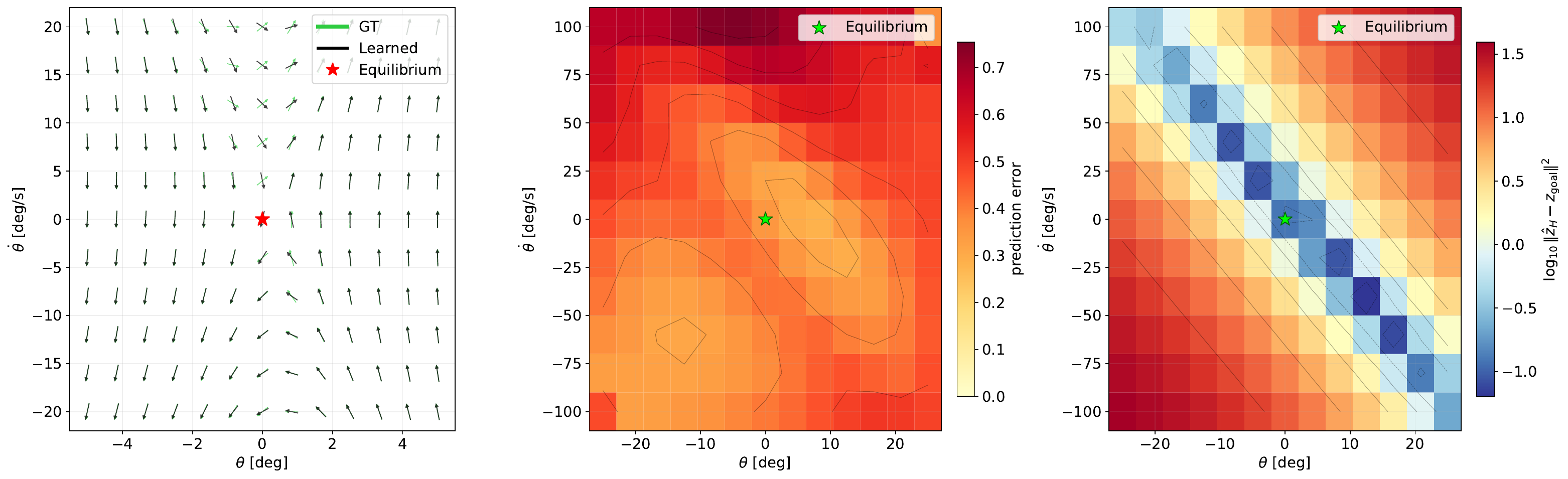}
\vspace{-0.5cm}
\caption{%
MSP+IBOT model on \textbf{CartPole}. Panel descriptions are the same as those in Fig.~ \ref{fig:landscape_MSP_EP-IDM_SIG}.
} 
\label{fig:landscape_MSP_IBOT}
\end{figure}

\begin{figure}[t]
\centering
\includegraphics[width=1\textwidth]{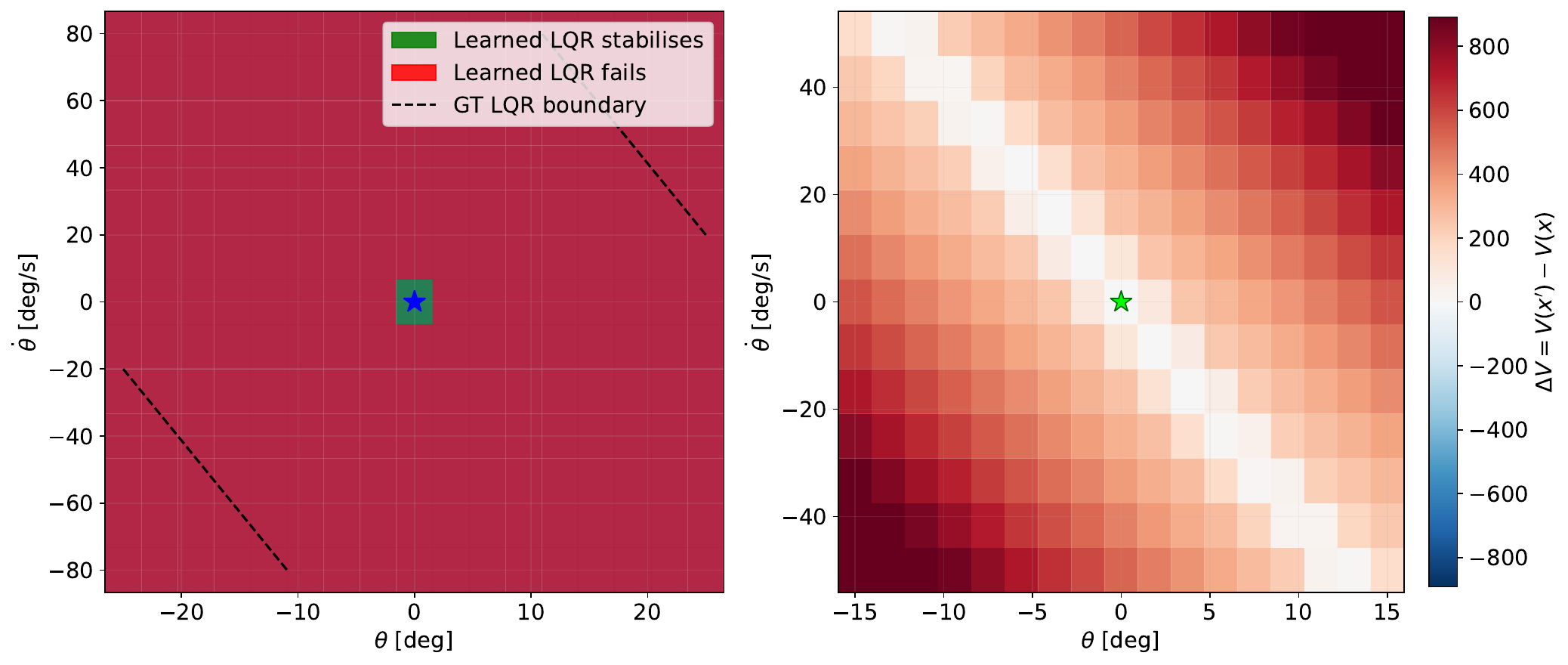}
\vspace{-0.5cm}
\caption{%
\small Local stability of MSP+IBOT on \textbf{CartPole}. \textbf{Left:} Empirical region of attraction (ROA) of the learned LQR policy in the GT environment. Panel descriptions are as those in Fig.~ \ref{fig:stability_MSP_EP_IDM_SIG}.}
\label{fig:stability_MSP_IBOT}
\end{figure}

\begin{figure}[t]
\centering
\includegraphics[width=1\textwidth]{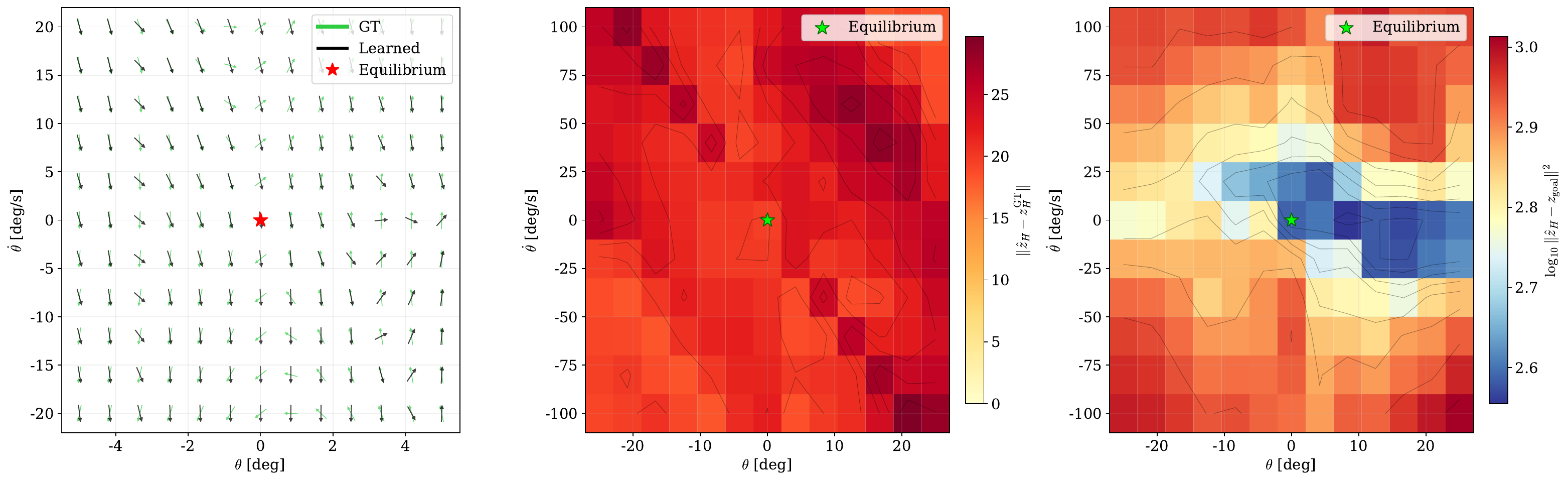}
\vspace{-0.5cm}
\caption{%
\small DINO-WM model on \textbf{CartPole}. Panel descriptions are the same as those in Fig.~ \ref{fig:landscape_MSP_EP-IDM_SIG}.
} 
\label{fig:landscape_DINOWM}
\end{figure}

Moreover, Figs.~\ref{fig:landscape_MSP_EP-IDM_SIG}--\ref{fig:stability_MSP_EP_IDM_SIG} show that the model trained from scratch with MSP+EP-IDM+SIG recovers the local vector field around the upright equilibrium, yields low prediction error around the equilibrium, and leads to a broad region of attraction with the expected decrease in the Lyapunov certificate. Figs.~\ref{fig:landscape_MSP_IBOT_PR_EP_IDM}--\ref{fig:stability_MSP_IBOT_PR_EP_IDM} depict similar performance when EP-IDM is applied through a learned projector on top of the pretrained iBOT. By comparison, the iBOT models without EP-IDM and DINO-WM have less accurate local dynamics, even when their terminal latent costs is minimized near the goal (see Figs.~\ref{fig:landscape_MSP_IBOT_PR_SIG}--\ref{fig:landscape_DINOWM}). 

\begin{table}[t]
\centering
\caption{\small CartPole CEM diagnostic metrics (action context 1).
$1/L^2$: inverse squared mean stabilization length.
FE: final error $\|x_T\|$.
PR denotes a learned projector on top of the frozen encoder.}
\label{tab:CartPole-cem-diagnostics}
\begin{tabular}{llcc}
\toprule
Encoder & Objective & CEM ($1/L^2$) & CEM (FE) \\
\midrule
\multicolumn{2}{l}{GT} & 0.0003 & 0.30202 \\
\midrule
From scratch
       & 1SP + SIG          & 0.0003 & 0.19956 \\
       & 1SP + IDM + MS-IDM & 0.0003 & 0.43608 \\
       & 1SP + EP-IDM       & 0.0003 & 0.13683 \\
       & MSP + SIG          & 0.0003 & 0.07667 \\
       & MSP + EP-IDM + SIG & 0.0003 & 0.11985 \\
\midrule
DINOv2 & 1SP             & 0.0038 & 4.60613 \\
       & MSP             & 0.0031 & 5.09216 \\
       & MSP + PR-SIG    & 0.0003 & 0.19056 \\
       & MSP + PR-IDM    & 0.0331 & 4.68120 \\
       & MSP + PR-EP-IDM & 0.0004 & 1.57507 \\
\midrule
iBOT   & 1SP             & 0.0003 & 0.09306 \\
       & 1SP + PR-EP-IDM & 0.0004 & 0.47275 \\
       & MSP             & 0.0073 & 3.49471 \\
       & MSP + PR-SIG    & 0.0004 & 1.42928 \\
       & MSP + PR-IDM    & 0.0567 & 5.11192 \\
       & MSP + PR-EP-IDM & 0.0003 & 0.62596 \\
\midrule
\multicolumn{2}{l}{DINO-WM} & 0.0005 & 0.9777 \\
\bottomrule
\end{tabular}
\end{table}

Table~\ref{tab:CartPole-cem-diagnostics} reports the diagnostics for the CEM experiments. The inverse squared stabilization length highlights models that lose stability early from those that remain close to the goal over the planning horizon, while the final error measures terminal accuracy. These results emphasize that CEM can compensate for imperfect local dynamics through repeated nonlinear replanning, whereas LQR directly depends on the local model around the equilibrium.

\subsubsection{PointMaze}
\label{appendix:pointmaze_additional}

PointMaze is open-loop stable and therefore does not require recovery of an unstable subspace. We, nevertheless, check whether IDM retains task-relevant directions without degrading nonlinear planning for open-loop stable tasks. Fig.~\ref{fig:pointmaze_trajectories} shows that both EP-IDM models and MSP+SIG reach the goal under CEM, whereas DINO-WM fails on the displayed rollout.

\begin{figure}[t]
\centering
\includegraphics[width=1\textwidth]{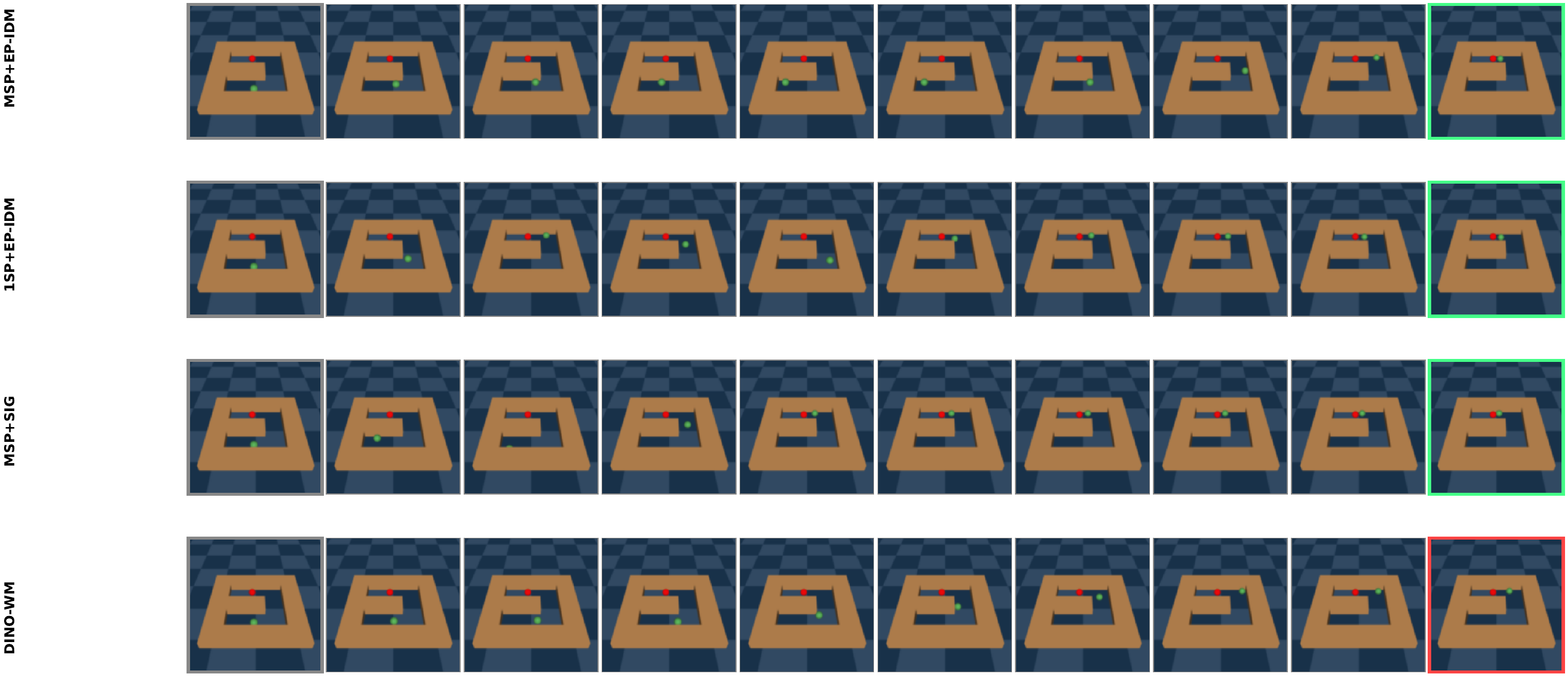}
\vspace{-0.5cm}
\caption{\small \textbf{PointMaze} visualization for our trained models under CEM planning.} 
\label{fig:pointmaze_trajectories}
\end{figure}

\begin{figure}[t]
\centering
\includegraphics[width=1\textwidth]{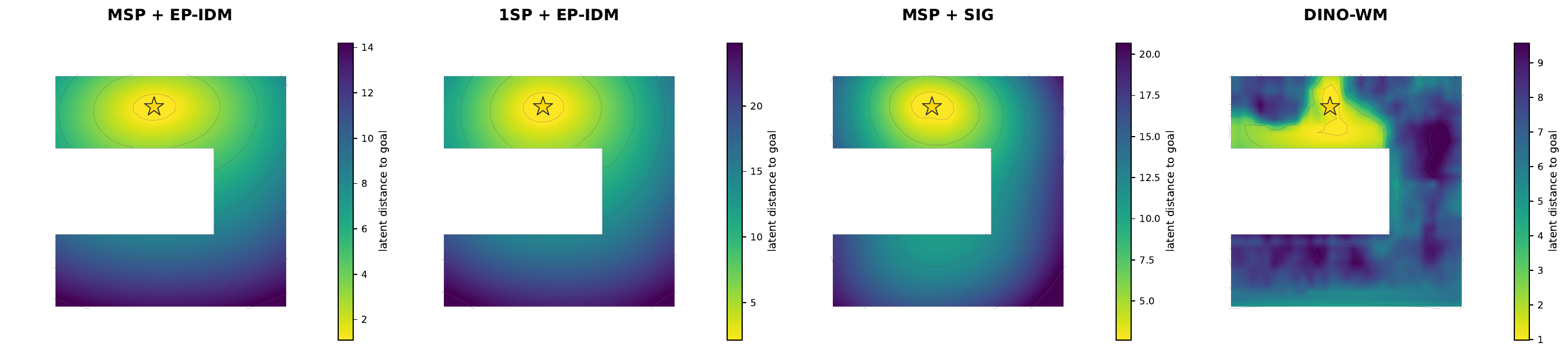}
\vspace{-0.5cm}
\caption{\small Latent goal-distance fields $\|z - z_{\mathrm{g}}\|$ across the \textbf{PointMaze} U-maze for four trained world models.} 
\label{fig:pointmaze_landscape}
\end{figure}

The latent goal-distance fields in Fig.~\ref{fig:pointmaze_landscape} explain this behavior. The models trained from scratch produce smooth fields that decrease toward the goal and broadly respect the maze geometry, providing CEM with an informative terminal cost. The DINO-WM field is less aligned with the navigation geometry, which can direct the planner toward states that appear close in latent space but are not connected by a feasible path. Importantly, on this open-loop stable task, the goal is to shape the representation to obtain a smooth planning latent geometry rather than the recovery of unstable modes.

\subsubsection{Closed-Loop Prediction Error}
\label{appendix:closed_loop_prediction_error}

\begin{wrapfigure}{r}{0.65\textwidth}
\centering
\vspace{-0.4cm}
\includegraphics[width=0.55\textwidth]{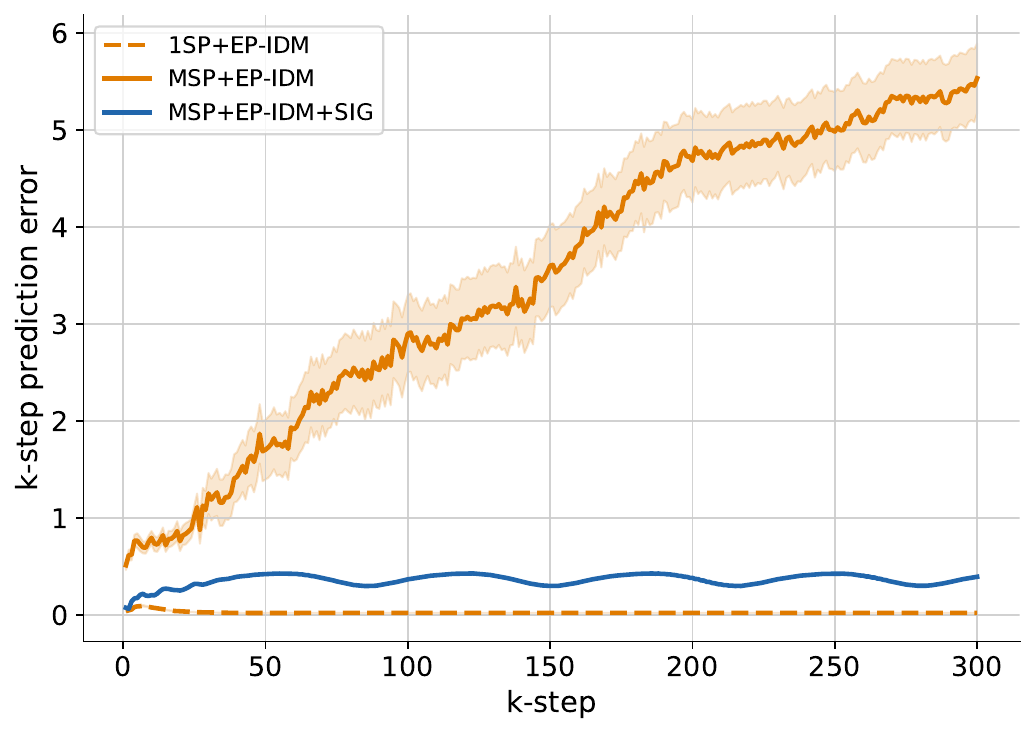}
\caption{\small $k$-step closed-loop prediction error along trajectories generated by the ground-truth LQR controller. Curves show the physical-state error between the ground-truth trajectory and the trajectory decoded from each latent rollout.}
\label{fig:k-step-error}
\vspace{-0.3cm}
\end{wrapfigure}

Figure~\ref{fig:k-step-error} compares prediction errors over long closed-loop rollouts. The 1SP+EP-IDM model is close to the ground-truth trajectory over the entire horizon. The error of MSP+EP-IDM grows steadily, whereas combining multi-step prediction and EP-IDM with SIGReg keeps the error bounded. Hence, multi-step training alone is not able to alleviate the rollout drift in an unstable system, i.e., small model errors can still be amplified by the dynamics. On the other hand, the representation regularization SIGReg improves the conditioning of the representation. This is complementary to the stability results in Section~\ref{sec:numerics}. The low rollout error in the landscape diagnosis supports control, but does not by itself guarantee that the representation preserves the unstable directions required by LQR.

\subsection{Pre-trained encoders}\label{appendix:pretrained}

We evaluate whether IDM can preserve control-relevant information when the visual encoder is frozen and only a lightweight projector (PR) and the latent dynamics are trained.

Table~\ref{tab:pre-trained-sigreg-idm} shows that the benefit of EP-IDM depends on the information and local geometry already present in the frozen latents. In particular, with iBOT, MSP+PR-EP-IDM achieves $100\%$ LQR and CEM success. In contrast, none of the DINOv2 configurations yields a stabilizing LQR controller, although MSP+PR-SIG supports CEM. Hence, IDM can shape the projector to preserve control-relevant information captured by the frozen backbone, but cannot recover information that the backbone has discarded.

\begin{table}[t]
\centering
\caption{\small\textbf{CartPole} control performance with frozen pretrained
encoders. SR: success rate. MFS: mean fraction of steps within the success
threshold. PR denotes a learned projector. DINO-WM \citep{zhou2024dino}
is an external baseline.}
\label{tab:pre-trained-sigreg-idm}
\setlength{\tabcolsep}{3.5pt}
\begin{tabular}{@{}llcccc@{}}
\toprule
Encoder & Objective & LQR (SR) & LQR (MFS) & CEM (SR) & GBP (SR) \\
\midrule
DINOv2
    & 1SP             & 0\% & 0.301 & 0\% & 0\% \\
    & MSP             & 0\% & 0.232 & 0\% & 0\% \\
    & 1SP + PR-EP-IDM & 0\% & 0.025 & 20\% & 0\% \\
    & MSP + PR-SIG    & 0\% & 0.460 & \textbf{100\%} & 0\% \\
    & MSP + PR-IDM    & 0\% & 0.455 & 0\% & 0\% \\
    & MSP + PR-EP-IDM & 0\% & 0.513 & 40\% & 10\% \\
\midrule
iBOT
    & 1SP             & 0\% & 0.249 & \textbf{100\%} & 80\% \\
    & MSP             & 0\% & 0.228 & 0\% & 0\% \\
    & 1SP + PR-EP-IDM & 0\% & 0.263 & \textbf{100\%} & 80\% \\
    & MSP + PR-SIG    & 50\% & 0.571 & 80\% & 80\% \\
    & MSP + PR-IDM    & 0\% & 0.405 & 0\% & 0\% \\
    & MSP + PR-EP-IDM & \textbf{100\%} & \textbf{0.998}
                      & \textbf{100\%} & 0\% \\
\midrule
\multicolumn{2}{l}{DINO-WM}
                      & 0\% & 0.019 & 40\% & 10\% \\
\bottomrule
\end{tabular}
\end{table}

Figure~\ref{fig:pretrained_LQR} depicts the closed-loop LQR trajectories. For every DINOv2-based model, the CartPole state norm crosses the stabilization threshold and grows without bound, except for MSP for which $\|x_t\|_2 \leq 10$. In contrast, the latent controller obtained with MSP+iBOT+PR-EP-IDM drives the state norm below the threshold and keeps it near the upright equilibrium for the full $300$-step rollout.

\begin{figure}[t]
\centering
\includegraphics[width=1\textwidth]{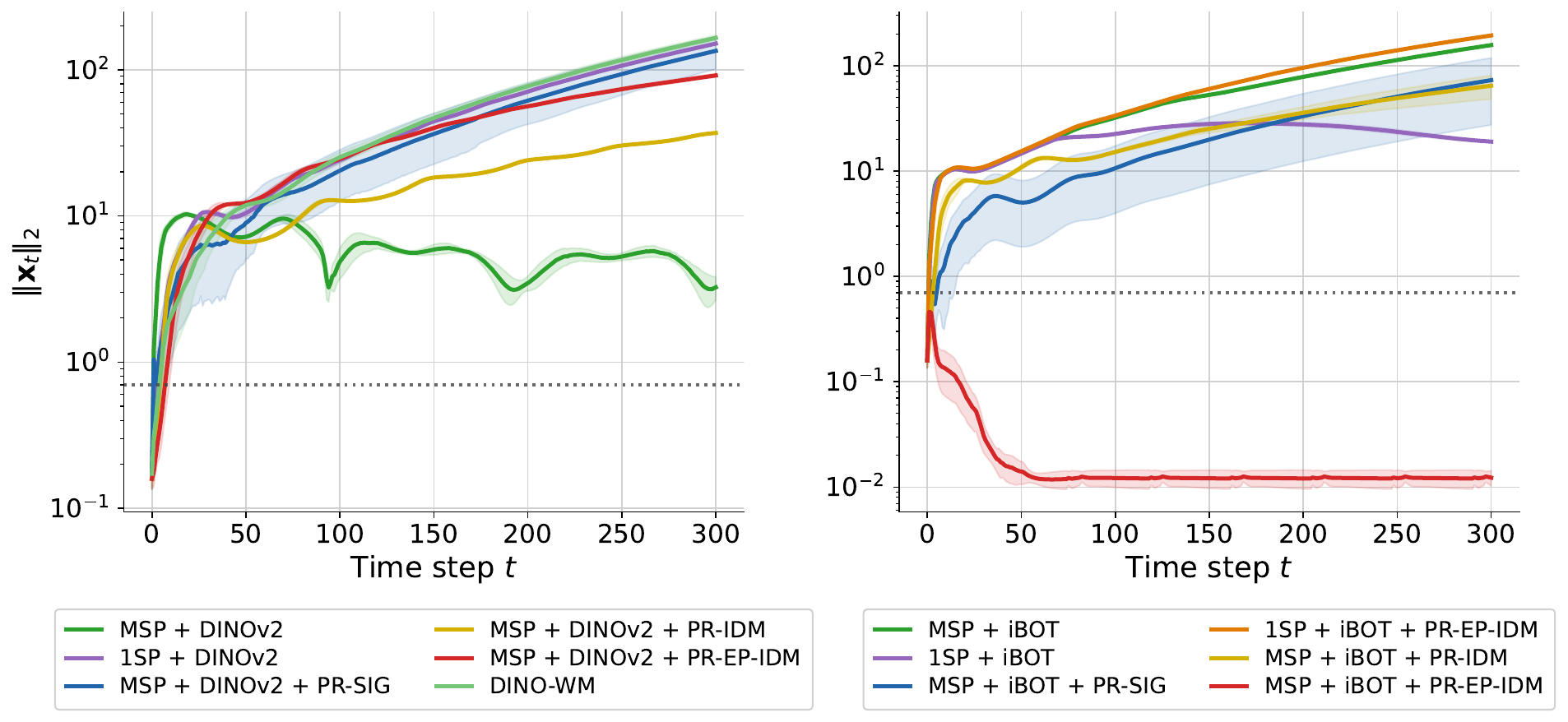}
\vspace{-0.6cm}
\caption{\small \textbf{CartPole} state norm $\|x_t\|_2$ under the learned latent LQR
controller over $300$ steps. \textbf{Left:} DINOv2. \textbf{Right:} iBOT.}
\label{fig:pretrained_LQR}
\end{figure}


\subsection{Additional Details on the Implementation of our Experiments}
\label{appendix:experimental_details}

In this section, we provide additional  implementation details needed to reproduce the nonlinear visual-control experiments from Section \ref{sec:numerics}. We first detail our world-model architecture, then specify the datasets and control parameters used for each task.

\subsubsection{Model Architecture}

We next describe the encoder, latent predictor, and action decoders used across our experiments.

\noindent \textbf{Visual encoder.} For models trained from scratch, each RGB observation is divided into non-overlapping patches and processed by a ViT-Tiny encoder with four transformer blocks, three attention heads, and embedding dimension $192$, similar to \cite{ivashkov2026sensorimotor}. The resulting class token is concatenated with the proprioceptive measurements and mapped to a $192$-dimensional latent state by a two-layer projection head. Unless stated otherwise, we encode the difference between consecutive frames together with the proprioceptive measurements to emphasize motion (i.e., velocity) of the physical system in the latent space.

For the frozen visual encoder, we use either the DINOv2 ViT-S/14 or iBOT ViT-S/16 class token. The backbone parameters are frozen and a trainable two-layer projector maps the concatenated visual and proprioceptive features to the common $192$-dimensional latent space. This construction follows the use of frozen visual features in DINO-WM \citep{zhou2024dino}.\vspace{0.2cm}

\noindent \textbf{Latent predictor and action decoders.} The latent predictor is an action-conditioned causal transformer with six blocks, $16$ attention heads, hidden dimension $192$, MLP dimension $2048$, and dropout $0.1$, also similar to the architecture adopted in \cite{ivashkov2026sensorimotor}. Actions are embedded by a two-layer MLP and modulate each transformer block through adaptive layer normalization. The prediction head is a two-layer MLP with hidden dimension $2048$ and output dimension $192$. 

The one-step and multi-step action decoders are MLPs with two hidden layers of width $256$ and ReLU nonlinearities. The endpoint decoder receives the initial and terminal latent states and uses the same hidden dimensions to reconstruct the complete length-$H$ action sequence. Then, all action-reconstruction variants differ only in the latent information provided to the decoder.

\begin{table}[H]
\centering
\caption{\small Task-specific observation and action dimensions. The frame skip is the number of simulator steps represented by one model transition.}
\label{tab:task_specific_architecture}
\begin{tabular}{lccccc}
\toprule
Task & Image size & Proprioception & Action & Frame skip & Training trajectories \\
\midrule
CartPole & $128\times128$ & $4$ & $1$ & $5$ & $787$ \\
Walker2D & $64\times64$ & $17$ & $6$ & $5$ & $3,200$ \\
PointMaze & $64\times64$ & $4$ & $2$ & $5$ & $1,600$ \\
\bottomrule
\end{tabular}
\end{table}

\noindent \textbf{Optimization.} We train all trainable modules jointly with AdamW. The learning rate is linearly warmed up during the first five epochs and then follows a cosine schedule. Gradients are clipped to unit norm. The coefficients of the active prediction, regularization, and action-reconstruction objectives are set to one. Table~\ref{tab:training_hyperparameters} reports the shared hyperparameters. All reported configurations use the same optimization budget.

\begin{table}[H]
\centering
\caption{\small World-model training hyperparameters.}
\label{tab:training_hyperparameters}
\begin{tabular}{lc@{\qquad}lc}
\toprule
Hyperparameter & Value & Hyperparameter & Value \\
\midrule
Optimizer & AdamW & Epochs & $200$ \\
Batch size & $64$ & Prediction horizon $H$ & $3$ \\
Initial learning rate & $10^{-4}$ & Minimum learning rate & $10^{-6}$ \\
Weight decay & $10^{-3}$ & Warm-up epochs & $5$ \\
Gradient-norm threshold & $1$ & Latent dimension & $192$ \\
Random seed & $42$ & Active loss weights & $1$ \\
\bottomrule
\end{tabular}
\end{table}

\subsubsection{Task and Dataset Details}

We next describe the data collection and evaluation protocols for each task. \vspace{0.2cm}

\noindent \textbf{CartPole.} We render the nonlinear CartPole system from a fixed camera and pair each image with the four-dimensional physical state. The dataset contains $985$ trajectories, split into $787/99/99$ training, validation, and test trajectories. To cover both uncontrolled and stabilizing behaviors, $75\%$ of the trajectories are collected using passive (i.e., $a_t = 0$ for all $t$) or uniformly random inputs, and the remaining $25\%$ use LQR-based policies and reference-angle sweeps. During evaluation, each method is tested from ten initial conditions for $300$ control steps. A trial is successful when the norm of the physical state remains below $0.7$ after $T$ steps. Actions are clipped to $[-10,10]$.
\vspace{0.2cm}

\noindent \textbf{Walker2D.} We use the Walker2D-v4 environment with $64\times64$ RGB observations, the $17$-dimensional proprioceptive measurements, and six-dimensional actions. The training data combine equal numbers of trajectories generated by a random policy and a pretrained SAC policy \citep{huang2022cleanrl,haarnoja2018soft}, yielding $3,200$ training, $400$ validation, and $400$ test trajectories after the split. Evaluation is conducted over ten trials of $600$ steps.
\vspace{0.2cm}

\noindent \textbf{PointMaze.} We use the U-shaped PointMaze environment with $64\times64$ RGB observations, four-dimensional proprioception, and two-dimensional actions. The dataset contains $2,000$ trajectories of length $100$, split into $1,600/200/200$ training, validation, and test trajectories. For each of ten trials, the point robot is initialized at a sampled start state and controlled to reach a trial-specific goal. Success happens when the planar distance to the goal is at most $0.5$.

\subsubsection{Planning and Control Details}

All planners optimize a terminal latent-space objective: the predicted terminal representation is matched to the representation of the desired goal observation. As in DINO-WM \citep{zhou2024dino}, CEM keeps a Gaussian distribution over action sequences, evaluates sampled sequences with the learned predictor, and select the lowest-cost elite actions, and refines the distribution. We use the iCEM variant in  \cite{pinneri2021sample} for Walker2D with a SAC policy warm initialization for the distribution of actions. Table~\ref{tab:cem_hyperparameters} provides the task-specific sampling budgets.

\begin{table}[H]
\centering
\caption{\small CEM and improved-CEM planning hyperparameters. Here, ``execute'' denotes the number of actions applied before replanning.}
\label{tab:cem_hyperparameters}
\small
\begin{tabular}{llccccc}
\toprule
Task & Planner & Horizon & Execute & Population & Elites & Iterations \\
\midrule
CartPole & CEM & $10$ & $1$ & $300$ & $30$ & $10$ \\
Walker2D &  iCEM & $3$ & $1$ & $200$ & $20$ & $5$ \\
PointMaze & CEM & $25$ & $25$ & $300$ & $30$ & $10$ \\
\bottomrule
\end{tabular}
\end{table}

For gradient-based planning (GBP), the action sequence is optimized directly through the differentiable latent rollout. We use Adam for $50$ optimization steps at each replanning instant, with learning rate $0.1$ and Gaussian action perturbations of standard deviation $0.05$. The CartPole and PointMaze horizons are $10$ and $25$, respectively.

For latent LQR, we linearize the learned latent dynamics around the desired equilibrium by automatic differentiation and solve the corresponding discrete algebraic Riccati equation. We use $Q=I_{192}$ in both tasks, with $R=1$ for CartPole and $R=I_2$ for PointMaze.

\subsubsection{Sketched Isotropic Gaussian Regularization  (SIGReg)} \label{appendix:SIGreg_loss}

For completeness, we define the SIGReg objective used to prevent total representation collapse by enforcing random one-dimensional projections of the latent distribution to match a standard isotropic Gaussian~\citep{balestriero2025lejepa}.

\begin{align*}
L_{\mathrm{SIG}}(\theta) =
\frac{1}{K}\sum_{\ell=1}^{K}\left(
\frac{\sum_{j=1}^{n_\omega} g(\omega_j)
\Big[\left(
\operatorname{Re}\hat\varphi_\ell(\omega_j)-\varphi_{\mathcal N}(\omega_j)
\right)^2+
\left(\operatorname{Im}\hat\varphi_\ell(\omega_j)\right)^2
\Big]
}{\sum_{j=1}^{n_\omega} g(\omega_j)
}\right).
\end{align*}
where
\begin{align*}
\hat\varphi_\ell(\omega)&=
\frac{1}{M}\sum_{i=1}^{M}
e^{\mathrm{i}\omega(w_\ell^\top z_i)},
\qquad \varphi_{\mathcal N}(\omega)
=e^{-\omega^2/2}
=:g(\omega),\qquad w_\ell
\overset{\text{i.i.d.}}{\sim}
\mathrm{Unif}\bigl(\mathbb S^{d-1}\bigr).
\end{align*}
Here, $K$ is the number of random projection directions, and
$w_\ell\in\mathbb S^{d-1}$ is the $\ell$-th random unit direction in
latent space. Each $w_\ell$ is obtained by drawing an i.i.d. Gaussian
vector in $\mathbb R^d$ and normalizing it. Moreover,
$z_i=E_\theta(y_i)\in\mathbb R^d$ is the $i$-th latent vector in the
current minibatch, obtained by pooling every ``frame'' from every
window, for $i=1,\dots,M$. The total number of pooled latents is
$M=N(H+1)$, corresponding to all $H+1$ frames from each of the $N$
windows. Lastly, $\omega\in\mathbb R$ denotes the
characteristic-function argument, and
$\{\omega_j\}_{j=1}^{n_\omega}$ is a grid of $n_\omega$ points evenly
spaced in $[-\omega_{\max},\omega_{\max}]$.

\subsubsection{Additional Action Reconstruction Losses} \label{appendix:additinal_action_reconstruction}

In addition to the endpoint action reconstruction loss studied in Section~\ref{sec:idm}, we empirically consider two alternative inverse-dynamics objectives. These objectives differ only in the latent information made available to the action decoder.\\

\noindent \textbf{Multi-step action reconstruction.} The multi-step decoder receives the complete latent trajectory and jointly reconstructs all actions applied within the window:
\begin{align}
D_{\mathrm{MS}\text{-}\mathrm{IDM},\theta}(z_{t_\tau},\ldots,z_{t_\tau+H})
&=\hat{\mathbf a}_{\tau,H}\in\mathbb R^{mH},
\end{align}
where $\hat{\mathbf a}_{\tau,H}=[\hat a_{t_\tau}^\top\;\cdots\;\hat a_{t_\tau+H-1}^\top]^\top$. In the pixel-based experiments, $D_{\mathrm{MS}\text{-}\mathrm{IDM},\theta}$ is a multilayer perceptron that first concatenates the $H+1$ latent states and then predicts the entire length-$H$ action sequence in a single forward pass. In the linear-system examples, we use a linear decoder. The corresponding loss is
\begin{align}
\mathcal{L}_{\mathrm{MS}\text{-}\mathrm{IDM}}(\theta)
=\mathbb{E}_{\tau \sim \mathcal{D}}\left[\frac{1}{H}\sum_{t=0}^{H-1}
\left\|\left[D_{\mathrm{MS}\text{-}\mathrm{IDM},\theta}\bigl(z_{t_\tau},\ldots,z_{t_\tau+H}\bigr)\right]_t-a_{t_\tau+t}\right\|_2^2\right].
\end{align}
In contrast to the endpoint action reconstruction, this decoder can use the intermediate representations to infer each action from local changes along the trajectory.\\

\noindent \textbf{One-step action reconstruction.} The one-step decoder is shared across time and reconstructs each action from a pair of consecutive latent states:
\begin{align}
D_{\mathrm{IDM},\theta}(z_{t_\tau+t},z_{t_\tau+t+1})
&=\hat a_{t_\tau+t}\in\mathbb R^m.
\end{align}
In our implementation, $D_{\mathrm{IDM},\theta}$ is a multilayer perceptron (MLP) applied independently to every adjacent latent pair. Its loss is
\begin{align}
\mathcal{L}_{\mathrm{IDM}}(\theta)
=\mathbb{E}_{\tau\sim \mathcal{D}}\left[\frac{1}{H}\sum_{t=0}^{H-1}
\left\|D_{\mathrm{IDM},\theta}\bigl(z_{t_\tau+t},z_{t_\tau+t+1}\bigr)-a_{t_\tau+t}\right\|_2^2\right].
\end{align}
This objective provides a local inverse-dynamics signal at each transition and uses the same decoder parameters for all $k$. This one-step action-reconstruction objective is the empirical
finite-sample counterpart of the inverse-dynamics regularizer proposed
by \citet{ivashkov2026sensorimotor}.

We emphasize that, empirically, we consider endpoint (EP-IDM), multi-step (MS-IDM), and one-step (IDM) action reconstruction both separately and jointly. A common loss encompassing all variants of action reconstruction considered in this work is given by
\begin{align}
\bar{\mathcal L}_{\mathrm{IDM}}(\theta) =\lambda_{1}\mathcal{L}_{\mathrm{EP}\text{-}\mathrm{IDM}}(\theta)
+\lambda_{2}\mathcal{L}_{\mathrm{MS}\text{-}\mathrm{IDM}}(\theta)
+\lambda_{3}\mathcal{L}_{\mathrm{IDM}}(\theta).
\end{align}
The theoretical guarantees in Section~\ref{sec:idm} leverage the endpoint loss ${\mathcal{L}}_{\mathrm{EP}\text{-}\mathrm{IDM}}$, whereas the two alternatives above are also included as empirical comparisons.

\subsection{Ablations}
\label{appendix:ablations}

We ablate the action reconstruction formulation, action-context length, and observation inputs to identify the effect of these components in the obtained control performance.
\vspace{0.2cm}

\noindent \textbf{Different inverse dynamics loss formulations.} In Table~\ref{tab:idm_variant_ablation}, we compare the different variants for action reconstruction, i.e., EP-IDM, IDM, and MS-IDM, used in the main experiments (see Appendix \ref{appendix:additinal_action_reconstruction} for additional details). We note that EP-IDM and the joint one-step/multi-step objective both produce perfect CartPole success with all three controllers. Adding SIGReg to multi-step prediction and endpoint reconstruction does not degrade performance.\vspace{0.2cm}

\begin{table}[t]
\centering
\caption{\small CartPole ablation over inverse-dynamics formulations for an encoder trained from scratch.}
\label{tab:idm_variant_ablation}
\begin{tabular}{lcccc}
\toprule
Objective & LQR (SR) & LQR (MFS) & CEM (SR) & GBP (SR) \\
\midrule
1SP + IDM + MS-IDM & 100\% & 0.999 & 100\% & 100\% \\
1SP + EP-IDM & 100\% & 0.999 & 100\% & 100\% \\
MSP + EP-IDM + SIG & 100\% & 0.996 & 100\% & 100\% \\
\bottomrule
\end{tabular}
\end{table}

\noindent \textbf{Action context.} Moreover, in Tables~\ref{tab:cartpole_context_comparison} and~\ref{tab:pointmaze_context_comparison}, we compare training with different action context lengths. Here, by action context, we mean the number of consecutive actions provided to the latent predictor when predicting the next latent. Here we consider action context one (the standard in our experiments) and five (the ablation) while holding the encoder and objective fixed. The trained-from-scratch action-reconstruction models do not benefit from more actions in the context length. However, the models trained with pretrained encoders indicate that a longer action context can actually help shape the representation for a better planning when using iBOT as the frozen encoder. 
\vspace{0.2cm}

\begin{table}[t]
\centering
\caption{\small \textbf{CartPole} success rates (\%) with action contexts one and five. PR denotes a learned projector on top of the frozen encoder.}
\label{tab:cartpole_context_comparison}
\small
\setlength{\tabcolsep}{4pt}
\begin{tabular}{@{}llcccccc@{}}
\toprule
& & \multicolumn{2}{c}{LQR (SR)}
  & \multicolumn{2}{c}{CEM (SR)}
  & \multicolumn{2}{c}{GBP (SR)} \\
\cmidrule(lr){3-4}
\cmidrule(lr){5-6}
\cmidrule(lr){7-8}
Encoder & Objective
& Context 1 & Context 5
& Context 1 & Context 5
& Context 1 & Context 5 \\
\midrule
From scratch
    & 1SP + SIG
    & 0 & 0 & 100 & 100 & 90 & 90 \\
    & 1SP + IDM + MS-IDM
    & 100 & 100 & 100 & 100 & 100 & 100 \\
    & 1SP + EP-IDM
    & 100 & 100 & 100 & 100 & 100 & 100 \\
    & MSP + SIG
    & 0 & 40 & 100 & 100 & 100 & 100 \\
\midrule
DINOv2
    & MSP
    & 0 & 0 & 0 & 0 & 0 & 0 \\
DINOv2
    & MSP + PR-SIG
    & 0 & 0  & 100 & 60 & 0 & 0 \\
\midrule
iBOT
    & MSP
    & 0 & 0 & 0 & 100 & 0 & 60 \\
    & MSP + PR-SIG
    & 50 & 40 & 80 & 90 & 80 & 60 \\
\bottomrule
\end{tabular}
\end{table}

\begin{table}[t]
\centering
\caption{\small \textbf{PointMaze} success rates (\%) with action contexts one and five.}
\label{tab:pointmaze_context_comparison}
\small
\setlength{\tabcolsep}{4pt}
\begin{tabular}{@{}llcccccc@{}}
\toprule
& & \multicolumn{2}{c}{LQR (SR)}
  & \multicolumn{2}{c}{CEM (SR)}
  & \multicolumn{2}{c}{GBP (SR)} \\
\cmidrule(lr){3-4}
\cmidrule(lr){5-6}
\cmidrule(lr){7-8}
Encoder & Objective
& Context 1 & Context 5
& Context 1 & Context 5
& Context 1 & Context 5 \\
\midrule
From scratch
    & 1SP + IDM
    & 70 & 80 & 100 & 90 & 100 & 80 \\
    & 1SP + EP-IDM
    & 80 & 80 & 100 & 100 & 80 & 80 \\
    & 1SP + IDM + MS-IDM
    & 80 & 70 & 90 & 100 & 90 & 90 \\
    & MSP + SIG
    & 60 & 80 & 90 & 70 & 100 & 60 \\
    & MSP + MS-IDM
    & 80 & 80 & 100 & 100 & 90 & 90 \\
    & MSP + EP-IDM
    & 80 & 80 & 90 & 90 & 90 & 100 \\
\midrule
DINOv2
    & MSP
    & 20 & 10 & 50 & 60 & 80 & 50 \\
\midrule
iBOT
    & MSP
    & 80 & 80 & 100 & 80 & 80 & 60 \\
\bottomrule
\end{tabular}
\end{table}

\noindent \textbf{Observation inputs.} In Table~\ref{tab:cartpole_state_only} we summarize the results when we remove the pixel observations (i.e., images) and compares MLP and transformer predictors using only the physical state (proprioceptive measurements). We note that EP-IDM succeeds with the transformer but not with the MLP under LQR, showing that representation sufficiency does not replace the need to fit an accurate local predictor. Moreover, in Table~\ref{tab:cartpole_proprioception}, we have the results when we remove the  proprioceptive measurements (states) from the frozen-iBOT models. We note that performance drops to zero across all planners, indicating that the present image encoder and short temporal context do not reliably recover velocity information from pixels alone.

\begin{table}[htbp]
\centering
\caption{\small \textbf{CartPole} ablation: Action context one, only proprioceptive
information (state).}
\label{tab:cartpole_state_only}
\small
\setlength{\tabcolsep}{4pt}
\begin{tabular}{@{}llccccc@{}}
\toprule
Predictor & Objective
& LQR (SR)
& LQR (MFS)
& CEM (SR)
& CEM (held)
& GBP (SR/held) \\
\midrule
MLP
    & MSP + SIG
    & 20\% & 0.571 & 20\% & 20\% & 0\% / 0\% \\
    & 1SP + EP-IDM
    & 0\% & 0.204 & \textbf{100\%} & \textbf{100\%}
    & 0\% / 0\% \\
\midrule
Transformer
    & MSP + SIG
    & 20\% & 0.461 & 0\% & 0\% & 0\% / 0\% \\
    & 1SP + EP-IDM
    & \textbf{80\%} & 0.869
    & \textbf{100\%} & \textbf{100\%}
    & \textbf{100\%} / \textbf{100\%} \\
\bottomrule
\end{tabular}
\end{table}

\begin{table}[htbp]
\centering
\caption{\small \textbf{CartPole} ablation with a frozen iBOT encoder and action
context five, with and without proprioceptive inputs. PR denotes the learned projector. SR: success rate, held: fraction of trials in which stabilization is preserved.}
\label{tab:cartpole_proprioception}
\small
\setlength{\tabcolsep}{4pt}
\begin{tabular}{@{}llccccc@{}}
\toprule
Objective & Proprioception
& LQR (SR)
& CEM (SR)
& CEM (held)
& GBP (SR)
& GBP (held) \\
\midrule
MSP
& Yes
& 0\% & \textbf{100\%} & \textbf{100\%} & \textbf{60\%} & \textbf{40\%} \\
& No
& 0\% & 0\% & 0\% & 0\% & 0\% \\
\midrule
MSP + PR-SIG
& Yes
& \textbf{40\%} & \textbf{90\%} & \textbf{90\%} & \textbf{60\%} & \textbf{60\%} \\
& No
& 0\% & 0\% & 0\% & 0\% & 0\% \\
\midrule
MSP + PR-IDM 
& Yes
& \textbf{100\%} & \textbf{80\%}  & \textbf{80\%} & \textbf{70\%} & \textbf{70\%}  \\
& No
& 0\% & 0\% & 0\% & 0\% & 0\% \\
\bottomrule
\end{tabular}
\end{table}

\noindent \textbf{Fixed-point consistency and encoder sensitivity.} We also evaluate how the two frozen encoders respond to small physical perturbations in order to understand the reason why MSP+PR-EP-IDM with iBOT succeeds and with DINOv2 fails. Table~\ref{tab:encoder-gain} reports the ratio between consecutive latent and physical displacements for zero-action rollouts near the upright equilibrium. iBOT has a gain close to one across the evaluated perturbations, whereas DINOv2 amplifies the same perturbations by approximately $22$-$26$ times. This is also consistent with the fixed-point residuals in Table~\ref{tab:fp-error}, namely, the predictor has difficulty representing the upright equilibrium state as a stationary latent point when small physical perturbations lead to disproportionately large latent perturbations. 

We also note that DINOv2 \citep{oquab2023dinov2} is trained on ImageNet \citep{deng2009imagenet} to maximally separate image representations, that is, it amplifies fine visual perturbations. In particular, a $0.02$-unit physical perturbation of the CartPole state changes the image by only a few pixels, but DINOv2's CLS token will completely separate it in the latent space. On the other hand, iBOT \citep{zhou2021ibot}, trained with masked reconstruction, learns a smoother representation where nearby images map to close latents.

\begin{table}[htbp]
\centering
\caption{\small Ablation on the encoder gain $\|\Delta z_t\| / \|\Delta x_t\|$ as a function of initial-state perturbation magnitude $\|x_0\|$.  We start from a perturbed upright equilibrium, the cart pole is stepped with zero action for 60 steps. We measure the mean latent displacement $\|\Delta z_t\| = \|z_t - z_{t-1}\|$ divided by the mean physical-state displacement $\|\Delta x_t\| = \|x_t - x_{t-1}\|$ . 
}
\label{tab:encoder-gain}
\begin{tabular}{lcccccc}
\toprule
Encoder & $\|x_0\|{=}0.005$ & $\|x_0\|{=}0.01$ & $\|x_0\|{=}0.02$
      & $\|x_0\|{=}0.05$ & $\|x_0\|{=}0.10$ & $\|x_0\|{=}0.20$ \\
\midrule
iBOT   & \textbf{1.07$\times$} & \textbf{1.11$\times$} & \textbf{1.07$\times$}
             & \textbf{1.07$\times$} & \textbf{1.02$\times$} & \textbf{1.03$\times$} \\
DINOv2 & 25.45$\times$ & 26.07$\times$ & 26.38$\times$
             & 25.67$\times$ & 21.83$\times$ & 25.12$\times$ \\
\bottomrule
\end{tabular}
\end{table}

\begin{table}[t]
\centering
\caption{\small Fixed-point residual of the learned predictor at the encoded
upright \textbf{CartPole} equilibrium under zero action. The residual measures the one-step drift from this encoded equilibrium. PR denotes a learned projector on top of the frozen encoder.}
\label{tab:fp-error}
\begin{tabular}{llc}
\toprule
Encoder & Objective
& $\|P(z_{\mathrm{eq}},0)-z_{\mathrm{eq}}\|$ \\
\midrule
DINOv2 & 1SP             & 5.569 \\
       & MSP             & 7.215 \\
       & 1SP + PR-EP-IDM & 0.4026 \\
       & MSP + PR-SIG    & 2.515 \\
       & MSP + PR-EP-IDM & 0.343 \\
\midrule
iBOT   & 1SP              & 0.059 \\
       & MSP              & 0.101 \\
       & 1SP + PR-EP-IDM  & 0.091 \\
       & MSP + PR-SIG     & 0.641 \\
       & MSP + PR-EP-IDM  & 0.129 \\
\bottomrule
\end{tabular}
\end{table}

\subsection{Additional Details on Examples of Section \ref{sec:stabilization_conditions}}
\label{appendix:example_details}

This section provides details on the system matrices, physical parameters, and linear camera maps used in the examples of Section \ref{sec:stabilization_conditions}. In addition, we also provide here an additional motivating example of a synthetic open-loop stable system.\vspace{0.2cm}

\noindent\textbf{Linear camera maps.}
In all three examples, we observe the physical state through a fixed linear perceptual ``sensor'' or camera map $y_t=Cx_t$. It lifts the low-dimensional physical state into a redundant observation space and mixes its coordinates, so that dynamically meaningful directions are not presented to the encoder in an axis-aligned basis. For a state dimension $n$ and observation dimension $p$, we draw $\bar C\in\mathbb R^{p\times n}$ using random seed zero, with $\bar C_{ij}\overset{\text{i.i.d.}}{\sim}\mathcal N(0,1)$, and normalize each column, i.e., $C_{:,j}=\bar C_{:,j}/\|\bar C_{:,j}\|_2$.\vspace{0.2cm}

\noindent\textbf{Synthetic, open-loop unstable system (Example 1).} This example isolates the collapse of unstable modes in the smallest possible setting: a two-dimensional, single-input system with one stable mode and one controllable but open-loop unstable mode that must be retained by the learned representation for feedback stabilization. We first define
\begin{align}
\bar{A}=\operatorname{diag}(1.25,0.85) \text{ and }
\bar{B}=\begin{bmatrix}1\\0.6\end{bmatrix}.
\end{align}
Both modes are therefore actuated by the single input. To avoid aligning the modal coordinates with the physical state coordinates, we apply the similarity transformation
\begin{align}
T_\vartheta=\begin{bmatrix}\cos\vartheta&-\sin\vartheta\\\sin\vartheta&\cos\vartheta\end{bmatrix},
\; \vartheta=0.6,\;  A=T_\vartheta \bar{A} T_\vartheta^{-1},
\text{ and } B=T_\vartheta \bar{B}.
\end{align}
This yields the following system matrices:
\begin{align}
A = \begin{bmatrix} 1.122 & 0.186 \\ 0.186 & 0.978 \end{bmatrix} \in \mathbb{R}^{2\times 2} \text{ and }
B = \begin{bmatrix} 0.487 \\ 1.060 \end{bmatrix} \in \mathbb{R}^{2\times 1}.
\end{align}
The six-dimensional linear camera map is
\begin{align}
C = \begin{bmatrix}
0.069 & -0.081 \\
0.353 & 0.064 \\
-0.295 & 0.222 \\
0.718 & 0.581 \\
-0.388 & -0.776 \\
-0.343 & 0.025
\end{bmatrix} \in \mathbb{R}^{6 \times 2}.
\end{align}

This system has one unstable and one stable mode with eigenvalues and eigenvectors as follows:

\begin{align}
\lambda_u = 1.25 \ (\textbf{unstable}),\ v_u = \begin{bmatrix}0.825\\0.565\end{bmatrix}
\text{   and   }
\lambda_s = 0.85 \ (\textbf{stable}),\ v_s = \begin{bmatrix}-0.565\\0.825\end{bmatrix}.
\end{align}

The initial condition is given by $x_0 = \alpha_u v_u + \alpha_s v_s$, where $\alpha_u \overset{\text{i.i.d.}}{\sim} \mathcal N(0,\sigma_u^2)$ with $\sigma_u = 0.01$ and $\alpha_s \overset{\text{i.i.d.}}{\sim} \mathcal N(0,\sigma_s^2)$ with $\sigma_s = 0.3$. We use an asymmetric initial-condition distribution, with larger variance along the stable mode and smaller variance along the unstable mode. We note that this choice reflects realistic data collection near an unstable equilibrium point \citep{levine2020offline}. That is, a policy that avoids catastrophic failure will naturally spend most of its time exploring near-equilibrium conditions or stable-mode excursions and will therefore rarely encounter the large deviations that reveal unstable directions. 

Each rollout has length $T=5$. With probability $0.3$, an episode is passive, meaning that $a_t \equiv 0$. Otherwise, the actions are sampled independently according to $a_t \overset{\text{i.i.d.}}{\sim} \mathcal N(0,\sigma_a^2)$, where $\sigma_a = 0.02$.\vspace{0.2cm}

\noindent\textbf{Synthetic, open-loop stable system (Example 2).} This example provides an open-loop stable control case for the previous example: it preserves the same observation, actuation, and data asymmetry as in the previous example while replacing the unstable mode by a stable mode, thereby separating representation collapse from its consequences for closed-loop stability.

\begin{align}
\lambda_{s1} = 0.25 \ (\textbf{stable 1}),\ v_{s1} = \begin{bmatrix}0.825\\0.565\end{bmatrix}
\text{   and   }
\lambda_{s2} = 0.85 \ (\textbf{stable 2}),\ v_{s2} = \begin{bmatrix}-0.565\\0.825\end{bmatrix}.
\end{align}

The initial-condition distribution, episode rollout length, per-step action distribution, and all training hyperparameters are kept the same as in the previous example.

In particular, we construct this example using the same $\bar{B}$, rotation $T_\vartheta$, and linear camera map as in the previous example, but replace the eigenvalue $1.25$ by $0.25$ in the diagonal of $\bar{A}$. We then have
\begin{align}
\bar{A}=\operatorname{diag}(0.25,0.85) \text{ and }
A=T_\vartheta \bar{A}T_\vartheta^{-1}
=\begin{bmatrix}0.441&-0.280\\-0.280&0.659\end{bmatrix},
\end{align}
while ${B}=[0.487\;1.060]^\top$ and $C$ is the $6\times2$ matrix reported above. Hence, the two examples differ only in whether the low-variance mode is unstable or contractive.

\begin{figure}[t]
\centering
\includegraphics[width=1\textwidth]{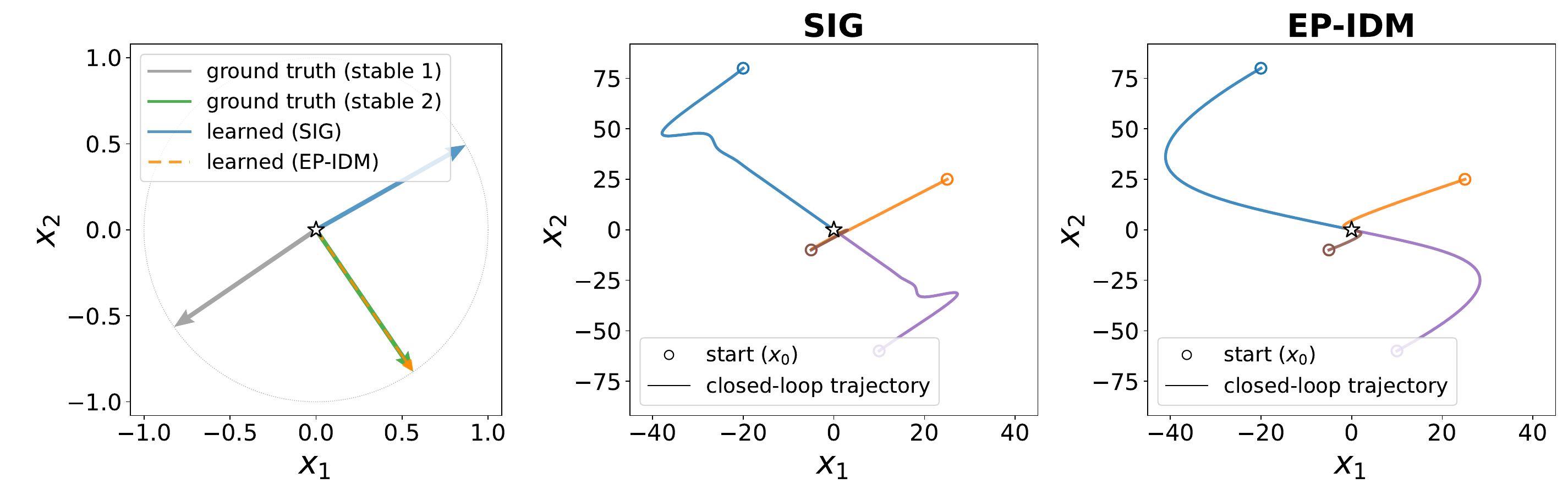}
\vspace{-0.8cm}
\caption{\small \textbf{Synthetic open-loop stable LDS.} This example considers a two-dimensional open-loop stable system. Panels descriptions are the same as those in Fig.~ \ref{fig:Examples_main}.} 
\label{fig:Example2}
\end{figure}

Figure~\ref{fig:Example2} provides an open-loop stable system baseline to compare with the examples in Fig. \ref{fig:Examples_main}, where all modes decay even without corrective feedback (i.e., they are all stable modes) and preserving a stable direction for latent stabilization is unnecessary. Hence, the standard latent world-model objective (i.e., solving Eq. (\ref{optimization:ms_SIGreg}) with SIGReg) can still work. For this open-loop stable system in Fig.~\ref{fig:Example2}, EP-IDM recovers the dominant stable direction.\vspace{0.2cm}

\noindent\textbf{Linearized CartPole (Example 3).} Let $p$ denote the cart position, $\theta$ the pole angle measured from the upright equilibrium, and $a$ the horizontal force applied to the cart. For the state $x=[p,\dot p,\theta,\dot\theta]^\top$, the linearized continuous-time dynamics are
\begin{align}
\dot x=A_cx+B_ca,
\end{align}
where
\begin{align}
A_c&=\begin{bmatrix}
0&1&0&0\\
0&0&-\frac{m g}{M}&0\\
0&0&0&1\\
0&0&\frac{(M+m)g}{M\ell}&0
\end{bmatrix} \text{ and }
B_c=\begin{bmatrix}0\\\frac{1}{M}\\0\\-\frac{1}{M\ell}\end{bmatrix}.
\end{align}
We use cart mass $M=1.0$, pole mass $m=0.1$, pole-length parameter $\ell=0.5$, and gravitational acceleration $g=9.81$. The continuous-time spectrum is $\{0,0,+\omega,-\omega\}$, where $\omega=\sqrt{(M+m)g/(M\ell)}$. We then discretize $(A_c,B_c)$ using an exact zero-order hold with sampling time $d_t=0.02$, which gives
\begin{align}
A = \begin{bmatrix}
1.000 & 0.020 & -0.0002 & -0.0000 \\
0.000 & 1.000 & -0.0196 & -0.0002 \\
0.000 & 0.000 & 1.0043 & 0.0200 \\
0.000 & 0.000 & 0.4323 & 1.0043
\end{bmatrix} \in \mathbb{R}^{4\times 4} \text{ and }
B = \begin{bmatrix} 0.0002 \\ 0.0200 \\ -0.0004 \\ -0.0401 \end{bmatrix} \in \mathbb{R}^{4\times1}.
\end{align}
We note that the underlying repeated eigenvalue at one corresponds to the cart's free-integrator block, while the upright pole contributes one stable and one unstable eigenvalue. Lastly, the eight-dimensional linear camera map is given by
\begin{align}
C = \begin{bmatrix}
0.046 & -0.068 & 0.256 & 0.056 \\
-0.195 & 0.185 & 0.522 & 0.503 \\
-0.256 & -0.647 & -0.250 & 0.022 \\
-0.847 & -0.112 & -0.499 & -0.389 \\
-0.198 & -0.162 & 0.165 & 0.553 \\
-0.047 & 0.699 & -0.266 & 0.187 \\
0.329 & 0.048 & -0.298 & -0.489 \\
-0.167 & 0.113 & -0.404 & -0.111
\end{bmatrix} \in \mathbb{R}^{8\times4}.
\end{align}

\begin{figure}[t]
\centering
\includegraphics[width=1\textwidth]{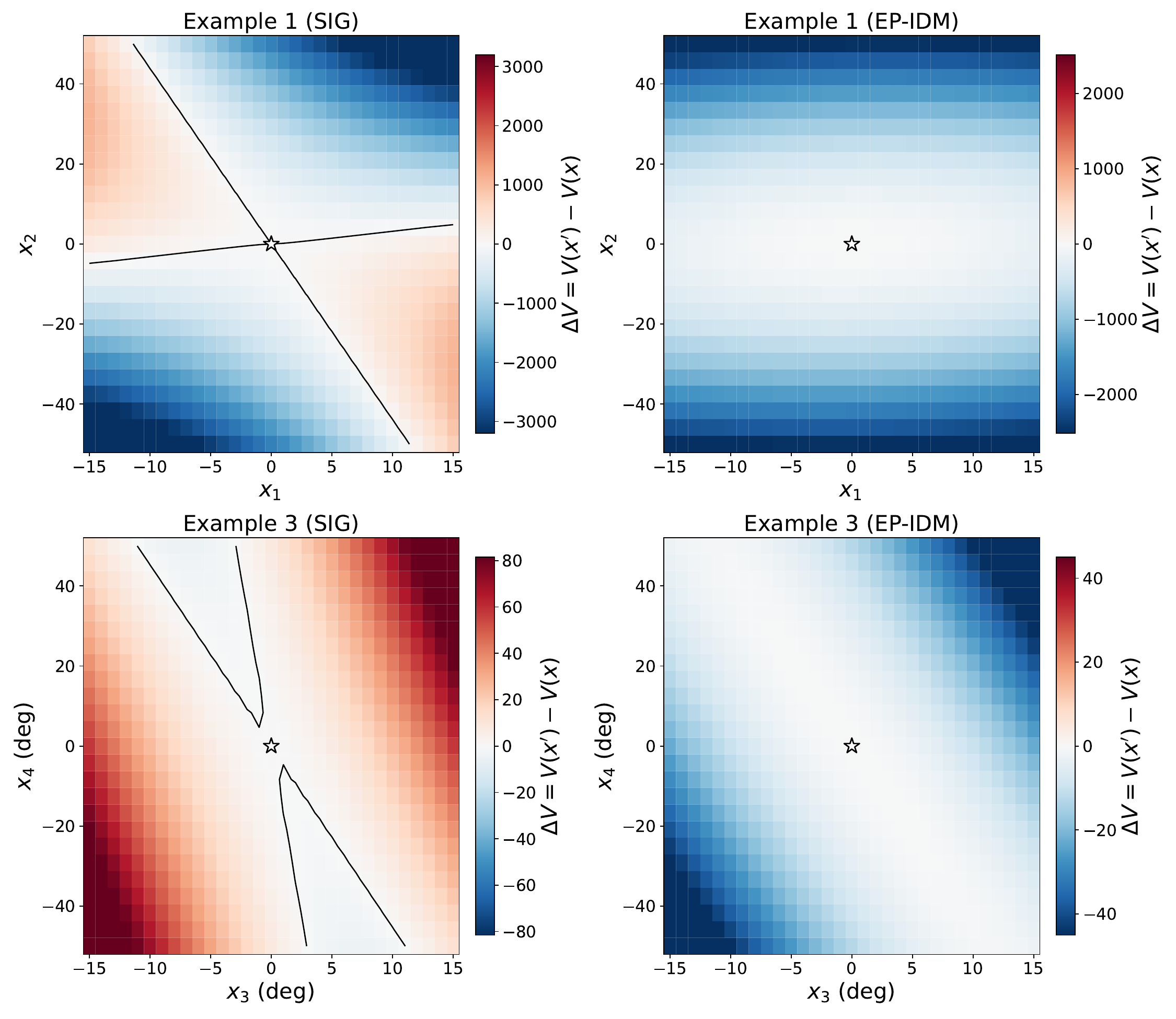}
\vspace{-0.5cm}
\caption{\small Sign of $\Delta V=V(x')-V(x)$ after one closed-loop step under the learned latent LQR controller, where $V(x)=x^\top S x$ and $S\succ 0$ is the
ground-truth Lyapunov function matrix. We also depict the $\Delta V=0$ contour (separatrix), i.e., the boundary separating states where the learned controller decreases $V$ (blue) from those where it does not (red).} 
\label{fig:Lyapunov}
\end{figure}

Fig.~\ref{fig:Lyapunov} complements the closed-loop control performance from Fig.~\ref{fig:Examples_main} with the Lyapunov certificate guarantee. For both unstable systems, the controllers learned with
SIGReg produce regions in which $\Delta V>0$, showing that the ground-truth
Lyapunov function increases after one closed-loop step under the learned latent LQR controller. This arises because the learned representations do not preserve the unstable dynamics required to design an accurate stabilizing feedback controller. In contrast, the controllers learned with EP-IDM yield $\Delta V<0$ throughout the evaluated region, except at the equilibrium where $\Delta V=0$.

\subsection{Proofs of Our Theoretical Guarantees}

This section provides the complete proofs of the theoretical results presented in the main text.\vspace{0.2cm}

\noindent \textbf{Notation.} Given a matrix $M\in\mathbb R^{n\times n}$, we denote  the spectral norm as $\|M\|_2$, the smallest singular value as $\sigma_{\min}(M)$, the spectral radius as $\rho(M)$, and the spectrum of $M$ as $\operatorname{spec}(M)$. The smallest and largest eigenvalues are $\lambda_{\min}(M)$ and $\lambda_{\max}(M)$. For a subspace $\mathcal S$, $P_{\mathcal S}$ denotes its orthogonal projector and $\mathcal S^\perp$ its orthogonal complement. We write $h \lesssim g $ when $h\leq Cg$ for a constant $C>0$.

\subsubsection{Supporting Lemmas}
\label{appendix:supporting_lemmas}

We first introduce the auxiliary results used throughout the subsequent proofs.

\begin{sublemma}\label{lemma:decoder_left_inverse}
Let $T$ be a linear map and suppose that the support of a random vector $u$ contains an open subset of its ambient space. If an affine map $D$ satisfies $D(Tu+c)=u$ almost surely, then $T$ has full column rank.
\end{sublemma}

\begin{proof}
Write $D(v)=Lv+b$. Exact reconstruction gives
\begin{align}
(LT-I)u+(Lc+b)=0
\end{align}
on an open set. An affine function that vanishes on an open set is identically zero, and hence $LT=I$. Therefore, $T$ admits a left inverse and has full column rank.
\end{proof}

\begin{theorem}[Davis-Kahan theorem \citep{davis1970rotation}]\label{theorem:davis_kahan}
Let $M$ and $\hat M=M+E$ be symmetric matrices. Let $\Phi$ be a subspace of $M$ associated with an eigenvalue cluster separated from the remaining spectrum of $M$ by a gap $\delta>0$, and let $\hat{\Phi}$ be the corresponding subspace of $\hat M$. Then, if $\|E\|_2\leq\delta/2$, it holds that
\begin{align}
\|\sin\Theta(\hat{\Phi},\Phi)\|_2
\leq \frac{2\|E\|_2}{\delta}.
\end{align}
\end{theorem}

\subsubsection{Proof of Lemma \ref{lemma:conditions_for_stabilization}}
\label{appendix:proof_conditions_for_stabilization}

Suppose first that $Av=\lambda v$, $|\lambda|\geq1$, and $Fv=0$. We consider two initial conditions that differ by $v$. As their initial representations agree, any deterministic controller using only the latent history applies the same first action to both systems. Inductively, the two latent histories, and hence the applied actions, remain identical because
\begin{align}
F A^t v=\lambda^tFv=0.
\end{align}
Their state difference is therefore $A^tv=\lambda^tv$, which cannot converge to zero, unless $v =0$. Thus, a single latent controller cannot stabilize both initial conditions. This proves necessity and is equivalent to the stated Popov-Belevitch-Hautus rank condition in \eqref{eq:pbh_detectability}.

On the other hand, suppose that $(A,F)$ is detectable. By Assumption \ref{assumption:stabilizable_observable}, $(A,B)$ is stabilizable, so there exists a feedback gain $K$ such that $A+BK$ is Schur stable. Moreover, the detectability guarantees the existence of an observer gain $L$ such that $A-LF$ is Schur stable. We consider the observer-based latent controller
\begin{align}
\hat x_{t+1}&=A\hat x_t+Ba_t+L(\tilde z_t-F\hat x_t),\notag\\
a_t&=K\hat x_t.
\end{align}
For the estimation error $e_t:=x_t-\hat x_t$, we have $e_{t+1}=(A-LF)e_t$. Moreover,
\begin{align}
x_{t+1}=(A+BK)x_t-BKe_t.
\end{align}
Thus, the joint dynamics of $(x_t,e_t)$ are block triangular with Schur-stable diagonal blocks $A+BK$ and $A-LF$. Hence, both $e_t$ and $x_t$ converge to zero. This demonstrates sufficiency and completes the proof.

\subsubsection{Proof of Theorem \ref{theorem:condition_achieved_action_reconstruction}}
\label{appendix:proof_action_reconstruction}

By unrolling the dynamics for $H$ steps, we obtain
\begin{align}
x_{t+H}=A^Hx_t+\mathcal C_H\mathbf a_t
\end{align}
where we recall that $\mathbf a_{t,H}=[a_t^\top\;\cdots\;a_{t+H-1}^\top]^\top$. Hence, we have
\begin{align}
z_{t+H}=FA^Hx_t+F\mathcal C_H\mathbf a_{t,H}+b.
\end{align}
Fix an initial state. The initial representation and the autonomous contribution to the final representation are then fixed (i.e., $FA^Hx_t+b$), while the only endpoint variation caused by the actions is $F\mathcal C_H\mathbf a_{t,H}$. As the linear decoder reconstructs $\mathbf a_{t,H}$ exactly on a conditional support containing an open set, Lemma \ref{lemma:decoder_left_inverse} implies that $F\mathcal C_H$ has full column rank.

Now we take $v\in\ker(F)\cap\mathcal R_H$. Therefore, there exists a vector $q$ such that $v=\mathcal C_Hq$. Hence, we have that
\begin{align}
F\mathcal C_Hq=Fv=0.
\end{align}
As $F\mathcal C_H$ has full column rank, $q=0$, and therefore $v=0$. This proves \eqref{eq:injective_reachable_subspace}. Therefore, if $\Phi_{\geq 1}\subseteq\mathcal R_H$, the same conclusion holds for every vector in $\Phi_{\geq 1}$. Finally, full column rank of $F\mathcal C_H\in\mathbb R^{d\times mH}$ requires $mH\leq\min\{d,n\}$.

\subsubsection{Proof of Theorem \ref{theorem:unstable_subspace_recovery}}
\label{appendix:proof_unstable_subspace_recovery}

For simplicity, we consider systems with no eigenvalues on the unit
circle. Thus, $\Phi_{\geq1}$ is the strictly unstable
subspace $\Phi$, i.e., $\Phi_{\geq1}:=\Phi$. In an orthonormal basis adapted to the right unstable invariant subspace $\Phi$ and its orthogonal complement $\Phi^\perp$, we write
\begin{equation}
A= \begin{bmatrix} A_u & \Delta\\ 0 & A_s \end{bmatrix}, \qquad B= \begin{bmatrix} B_u\\ B_s
\end{bmatrix}.
\end{equation}

We make the following assumptions.  For each $j\geq0$, let $\Gamma_j$ denote the unstable block of
$A^jB$, so that
\begin{equation}
A^jB=
\begin{bmatrix}
\Gamma_j\\
A_s^jB_s
\end{bmatrix}.
\end{equation}

\begin{assumption}\label{assumption:spectral_growth_separation}
There exist constants $\alpha>1$, $\beta\in(\rho(A_s),1)$, $\gamma\geq\alpha$, and $C_1,C_2,C_3\geq1$ such that
\begin{align}
\sigma_{\min}(A_u)&\geq\alpha,\;\;
\|A_s^j\|_2\leq C_1\beta^j,\;\;
\|A_u^j\|_2\leq C_2\gamma^j,\;\;
\|\Gamma_j\|_2\leq C_3\gamma^j, \forall j\geq0, \; \text{ and } \; \gamma\beta<\alpha^2.
\end{align} 
\end{assumption}

We emphasize that under
Assumptions~\ref{assumption:spectral_growth_separation}, there exists $\alpha>1$, $\beta\in(\rho(A_s),1)$, $\gamma\geq\alpha$, such that every direction in $\Phi$ expands by at least a factor $\alpha$ at each open-loop step, whereas the stable dynamics contract at rate $\beta$. The condition $\gamma\beta<\alpha^2$ ensures that the coupling between the unstable and stable responses grows strictly slower than the unstable contribution.

\begin{assumption}\label{assumption:unstable_reachability}
The pair $(A_u,B_u)$ is controllable, and the action transmitted through the coupling $\Delta$ does not  cancel the direct unstable input response. In particular, there exist $c_u>0$ and $H_0\geq1$ such that
\begin{align}
\lambda_{\min}(\Pi_{u,H})\geq c_u\alpha^{2H},\text{ for } H\geq H_0.
\end{align}
For block-diagonal dynamics, i.e., $\Delta\equiv0$, this lower bound follows directly from controllability of $(A_u,B_u)$ and Assumption \ref{assumption:spectral_growth_separation}.
\end{assumption}

The constants $c_u,c_s,$ and $c_{us}$ have a direct control-theoretic interpretation. The constant $c_u$ measures the smallest action-induced variation generated along any unstable direction, i.e., a larger $c_u$ means that even the least excitable unstable direction becomes distinguishable over a shorter
horizon. The constant $c_s$ bounds the total action-induced variation that can accumulate within the stable subspace, which remains bounded because the stable dynamics contract. Finally, $c_{us}$ quantifies the interaction between the action-induced unstable and stable responses. A larger $c_{us}$ allows stronger mixing between these responses at finite horizons, but Assumption \ref{assumption:spectral_growth_separation} ensures that this interaction grows more slowly than the unstable contribution $c_u\alpha^{2H}$.

Moreover, under Assumptions \ref{assumption:spectral_growth_separation} and \ref{assumption:unstable_reachability}, there exist constants $c_s,c_{us}>0$ such that, for all sufficiently large $H$, we have 
\begin{align}\label{eq:gramian_scaling_assumptions}
\lambda_{\min}(\Pi_{u,H})&\geq c_u\alpha^{2H}, \;
\|\Pi_{s,H}\|_2\leq c_s, \text{ and }
\|\Pi_{us,H}\|_2\leq c_{us}\sum_{j=0}^{H-1}(\gamma\beta)^j,
\end{align}
where $\Pi_{u,H}$, $\Pi_{s,H}$, and $\Pi_{us,H}$ are the unstable, stable, and cross Gramian blocks, respectively.

\begin{sublemma}\label{lemma:gramian_growth}
Suppose that the assumptions preceding Theorem \ref{theorem:unstable_subspace_recovery} hold. Then, the Gramian blocks satisfy
\begin{align}
\lambda_{\min}(\Pi_{u,H})&\geq c_u\alpha^{2H},\notag\\
\|\Pi_{s,H}\|_2&\leq c_s,\notag\\
\|\Pi_{us,H}\|_2&\leq c_{us}\sum_{j=0}^{H-1}(\gamma\beta)^j.
\end{align}
Therefore, the unstable block grows exponentially, the stable block remains bounded, and the cross block grows strictly more slowly than the unstable block whenever $\gamma\beta<\alpha^2$.
\end{sublemma}

\begin{proof}
The first inequality follows from Assumption \ref{assumption:unstable_reachability}. The remaining bounds follow from $\|\Gamma_j\|_2\leq C_3\gamma^j$ and $\|A_s^j\|_2\leq C_1\beta^j$. In particular,
\begin{align}
\|\Pi_{s,H}\|_2
&\leq \sum_{j=0}^{H-1}\|A_s^jB_s\|_2^2
\leq C_1^2\|B_s\|_2^2\sum_{j=0}^{\infty}\beta^{2j}
=\frac{C_1^2\|B_s\|_2^2}{1-\beta^2},\notag \\
\|\Pi_{us,H}\|_2
&\leq \sum_{j=0}^{H-1}\|\Gamma_j\|_2\|A_s^jB_s\|_2
\leq C_3C_1\|B_s\|_2\sum_{j=0}^{H-1}(\gamma\beta)^j.
\end{align}
The proof is completed when we absorb the constant factors into $c_s$ and $c_{us}$.
\end{proof}

\begin{definition}[Subspace distance] \label{definition:principal_angle_distance}Let $\Xi,\hat \Xi\in\mathbb R^{n\times r}$ be orthonormal bases for
two $r$-dimensional subspaces $\Phi$ and $\hat{\Phi}$, respectively. Their principal angles
$\theta_1,\dots,\theta_r\in[0,\pi/2]$ are defined by
\begin{align}
\sigma_i(\Xi^\top\hat \Xi)=\cos(\theta_i), \text{ for all } i=1,\dots,r,
\end{align}
where $\sigma_i(\cdot)$ denotes the $i$-th singular value. We define
\begin{align}
\sin\Theta(\hat{\Phi},\Phi):=\operatorname{diag}\bigl(\sin(\theta_1),\dots,\sin(\theta_r)\bigr).
\end{align}
The spectral-norm distance between the two subspaces is then
\begin{align}
\bigl\|\sin\Theta(\hat{\Phi},\Phi)\bigr\|_2:=
\max_{i=1,\dots,r}\sin(\theta_i)=\bigl\|P_{\hat{\Phi}}-P_{\Phi}\bigr\|_2,
\end{align}
where $P_{\hat{\Phi}}$ and $P_{\Phi}$ are the orthogonal projectors onto $\hat{\Phi}$ and $\Phi$, respectively. 
\end{definition}

We proceed by writing
\begin{align}
\Pi_H =\underbrace{\begin{bmatrix}\Pi_{u,H}&0\\0&0\end{bmatrix}}_{M_1}+\underbrace{\begin{bmatrix}0&\Pi_{us,H}\\\Pi_{us,H}^\top&\Pi_{s,H}\end{bmatrix}}_{M_2}.
\end{align}
The top $r$-dimensional eigenspace of $M_1$ is exactly $\Phi$, and its eigengap is at least $c_u\alpha^{2H}$. Moreover, Lemma \ref{lemma:gramian_growth} yields
\begin{align}
\|M_2\|_2
\leq \|\Pi_{us,H}\|_2+\|\Pi_{s,H}\|_2 \leq c_{us}\sum_{j=0}^{H-1}(\gamma\beta)^j+c_s.
\end{align}
As $\gamma\beta<\alpha^2$, the ratio $\|M_2\|_2/(c_u\alpha^{2H})$ converges to zero. It is therefore smaller than $1/2$ for all sufficiently large $H$. We then proceed, by applying Theorem \ref{theorem:davis_kahan} to write 
\begin{align}
\|\sin\Theta(\hat{\Phi},\Phi)\|_2
&\leq \frac{2\|M_2\|_2}{c_u\alpha^{2H}}\leq\frac{2\left(c_{us}\sum_{j=0}^{H-1}(\gamma\beta)^j+c_s\right)}{c_u\alpha^{2H}}.
\end{align}

It remains to make the convergence rate explicit. We then set $q=\gamma\beta$.  Under Assumptions~\ref{assumption:spectral_growth_separation}
and~\ref{assumption:unstable_reachability},
for all sufficiently large $H$, we can write
\begin{equation}
\bigl\|\sin\Theta(\hat{\Phi},\Phi)\bigr\|_2
\leq 2\left(c_{us}\displaystyle\sum_{j=0}^{H-1}
(\gamma\beta)^j+c_s\right)/\left(c_u\alpha^{2H}\right).
\end{equation}

The convergence rate is then given by the following three cases:
\begin{itemize}
\item If $q>1$, then it holds that
\begin{equation}
\bigl\|\sin\Theta(\hat{\Phi},\Phi)\bigr\|_2
\lesssim (q^H+1)/\alpha^{2H}.
\end{equation}

\item If $q=1$, then it holds that
\begin{equation}
\bigl\|\sin\Theta(\hat{\Phi},\Phi)\bigr\|_2
\lesssim H/\alpha^{2H}.
\end{equation}

\item If $q<1$, then it holds that
\begin{equation}
\bigl\|\sin\Theta(\hat{\Phi},\Phi)\bigr\|_2
\lesssim 1/\alpha^{2H}.
\end{equation}
\end{itemize}
As $q<\alpha^2$, each of these bounds converges to zero as
$H\to\infty$. Then, we have
\begin{equation}
\bigl\|\sin\Theta(\hat{\Phi},\Phi)\bigr\|_2 \to 0
\text{ as } H\to\infty,
\end{equation}
which completes the proof.

\end{document}